\documentclass[twoside]{article}

\usepackage{hyperref}
\usepackage{url}
\usepackage{amsmath}
\usepackage{amssymb}
\usepackage{mathtools}
\usepackage{amsthm}
\usepackage{thmtools}
\usepackage{thm-restate}
\usepackage{listings}
\usepackage{lstbayes}

\theoremstyle{plain}
\newtheorem{theorem}{Theorem}[section]

\newtheorem{lemma}[theorem]{Lemma}

\theoremstyle{definition}

\theoremstyle{remark}

\usepackage[accepted]{aistats2026}
\usepackage[round]{natbib}

\newcommand{\E}{\mathrm{E}}

\def\citeapos#1{\citeauthor{#1}'s (\citeyear{#1})}

\begin{document}

\twocolumn[

\aistatstitle{A Divergence-Based Method for Weighting and Averaging Model Predictions}

\aistatsauthor{ Olav Benjamin Vassend }

\aistatsaddress{University of Inland Norway} ]

\begin{abstract}
This paper uses a minimum divergence framework to introduce a new way of calculating model weights that can be used to average probabilistic predictions from statistical and machine learning models. The method is general and can be applied regardless of whether the models under consideration are fit to data using frequentist, Bayesian, or some other fitting method. The proposed method is motivated in two different ways and is shown empirically to perform better than or on a par with standard model averaging methods, including model stacking and model averaging that relies on Akaike-style negative exponentiated model weighting, especially when the sample size is small. Our theoretical analysis explains why the method has a small-sample advantage. \end{abstract}

\section{INTRODUCTION}

It is well known that averaging predictions from multiple models can improve predictive accuracy. Averaging model predictions generally requires that we assign a weight to each model that represents how accurate the model is likely to be on future observations. In addition to their role in model averaging, model weights are often used to compare the relative merits of competing models \citep{Wagenmakers2004}, or for assessing the importance of predictors \citep[p. 167]{Burnham1998}. This paper considers the problem of assigning weights to models within a minimum-divergence framework similar to those employed in \cite{williams1980, diaconis1982} and \cite{bissiri2016}, and proposes a simple method for doing so. Our perspective throughout is predictive: the objective is to construct a probability distribution over future outcomes that achieves high predictive accuracy, where accuracy is typically quantified using a strictly proper scoring rule \citep{gneiting2007}.

Although our main aim in this paper is to show that the proposed method often has a predictive advantage over competing approaches, especially in the small-sample case, the method has other strengths as well. First, in contrast to many competitors, it is not intrinsically computationally intensive. Second, the weights it produces tend to be more stable than those produced by other approaches in our experiments (see Section 4.1). More broadly, another aim of this article is to bridge several literatures that are not often connected and to shed light on when and why different model-weighting approaches work.

In brief, the method that will be proposed works as follows. Given $K$ models $\mathcal{M}_1, \mathcal{M}_2$, $\ldots, \mathcal{M}_K$, and  $n$ data points $y_1, y_2, \ldots, y_n$, let $p_k^p(y_i)$ be the probability (density) that model $k$ assigns to data point $y_i$ after the model has been fit to all of the data using some fitting method. Let $\tilde{y_1}, \tilde{y_2}, \ldots, \tilde{y_n}$ be a draw of $n$ new data points. We then define the ``optimism'' of each model as follows:

\begin{equation}
op_k  = -\sum_i \log p^p_k(\tilde{y}_i) + \sum_i\log p^p_k(y_i) \label{optimism}
\end{equation}

$op_k$ is the degree to which the in-sample accuracy overestimates the predictive accuracy of the model on future data. Note that it is a random variable since it depends on the future (unseen) data. We now define the following set of optimism-penalizing ``prior'' model weights:

\begin{equation}
w_k^{op} = \frac{e^{-op_k}}{\sum_{i=1}^K{e^{-op_i}}}\label{optimism_weights}
\end{equation}

These prior weights summarize the \emph{relative} optimism of all the models under consideration in the sense that models that are comparatively optimistic receive a lower probability and models that are comparatively pessimistic receive a higher probability. Finally, we consider the goal of  selecting optimal ``posterior'' model weights $w_k^p$. Formally, if $S^K$ is the set of probability mass distributions with $K$ components, we consider the following optimization problem: 

\begin{equation} \min_{w^p \in S^K} \sum_k w_k^{p}\log\frac{w_k^p}{w^{op}_k} - \sum_i \log{\sum _k  w_k^pp_k^p(y_i)}\label{parsimony_stacking} \end{equation}

The left-hand sum in (\ref{parsimony_stacking}) is the KL divergence from the posterior weights to the optimism-penalizing prior model weights whereas the right-hand sum is a measure of the predictive accuracy of the posterior weights on the data, so that (\ref{parsimony_stacking}) may be regarded as trading off divergence from the prior against divergence from the data. For lack of a better name, we therefore call posterior model weights that optimize the above expression ``divergence-based weights.'' (\ref{parsimony_stacking}) is a convex optimization problem in $w^p$ (the KL term is strictly convex and -log of an affine positive form is convex), hence it has a unique solution and can be solved efficiently with general-purpose optimizers, e.g., Solnp \citep{Rsolnp2015}.

In a Bayesian context, related optimization problems have been proposed by \cite{Masegosa2020}, \cite{Futami2021}, \cite{Futami2022}, and \cite{Morningstar2022}. However, these papers are concerned with Bayesian inference within a single model, whereas our goal is to average predictions from multiple (complex and often non-Bayesian) models, which is why we need a prior that penalizes model optimism. Crucially, the divergence-based model weighting optimization problem \ref{parsimony_stacking} does not fall under the general ``rule of three'' format proposed by \cite{Knoblauch2022}, because the fit-to-data term averages over the posterior inside rather than outside the loss function. Very recently, and independently of us, \cite{McLatchie2025} propose a variant of (\ref{parsimony_stacking}) that does average over the posterior inside rather than outside the loss function, but---in contrast to us---\cite{McLatchie2025} focus on predictive accuracy given a Bayesian model rather than averaging over multiple models. 

It is important to emphasize that the optimism-penalizing ``prior'' is not a Bayesian prior distribution reflecting a belief that each model is true. As we explain in the appendix, we think it is sensible to interpret the prior and posterior weights as reflecting a comparative degree of trust in the predictive accuracy of each model under consideration.  

Because model optimism as defined in (\ref{optimism}) involves future data, it must itself be estimated. A variety of approaches are available for this purpose, including cross-validation, the bootstrap, and information criteria such as AIC or WAIC. In order to make comparisons with competing methods as transparent and fair as possible, all experiments in this paper use 5-fold cross-validation (CV) to estimate optimism. That is, let $F_1, F_2, \ldots, F_5$ be a subdivision of the data into $5$ equally sized folds, let $-F_j$ consist of the data not in $F_j$, and let $p^{-F_j}_k$ denote model $k$ when fit to all data not in $F_k$. We then estimate model $k$'s optimism by: $\hat{op}_k = -\sum_{j=1}^5 \log p^{-F_j}_k(F_j) + \sum_j\log p^p_k(F_j)$. 

Although we consistently use 5-fold CV in this paper, we note that a significant advantage of our method is that we can avail ourselves of any estimate of model optimism. For example, the AIC penalty term is a good alternative if all the models are parametric and fit using maximum likelihood. It is significantly more stable and less computationally expensive than CV, and (somewhat surprisingly perhaps) in our experience it often yields better predictive accuracy with small sample sizes than CV.

For reasons of numerical stability, we solve the following equivalent version of (\ref{parsimony_stacking}) in practice:

\begin{equation} \min_{w^p \in S^K} \sum_k w^p_k\log w^p_k + \sum_k w_k^{p}op_k - \sum_i \log{\sum _k  w_k^pp_k^p(y_i)}\label{parsimony_stacking2} \end{equation}

The implementation of divergence-based model-weighting can then be summarized as follows: (i) use cross validation (or some other method) to estimate the optimism $op_k$ of each of the $K$ models under consideration. (ii) Fit each model to all the data $y_1, \ldots, y_n$ and form the matrix of pointwise predictions $p^p_k(y_i)$. (iii) Plug the estimates of $op_k$ and $p^p_k(y_i)$ into (\ref{parsimony_stacking2}) and solve for $w^p$.

The plan for the rest of the paper is as follows. Section 2 briefly introduces model stacking and negative exponentiated weighting. Section 3 presents two distinct theoretical justifications of divergence-based model weighting. Section 4 demonstrates the method in a linear regression simulation and on several data sets. Finally, Section 5 discusses various ways divergence-based model weighting may be extended and improved upon. The appendix shows how to implement the method in R \citep{RCoreTeam2020}, gives examples of how to use the method in a Bayesian context, gives a third justification of the method that shows how it may be regarded as a modification of Bayesian model averaging. The R code that was used to conduct all simulations and data analyses in the paper is available on GitHub at \url{https://github.com/Vassendo/DivergenceBasedModelWeighting}.

\section{EXISTING MODEL WEIGHTING AND AVERAGING METHODS}

If we look at existing model weighting and averaging methods, we can distinguish between two broad classes of popular approaches. One approach is to separately calculate a predictive score for each model and then transform these scores into model weights. Bayesian model averaging as well as model averaging based on information criteria, such as AIC-based model weighting  \citep{Akaike1979, Rigollet2012}, fall into this category. The other popular approach is to directly estimate the best model weights or the best combination of model predictions by creating a ``model ensemble'' \citep{wolpert1992, breiman1996, Yuling2018}. As we will see, both these approaches have drawbacks that the divergence-based procedure that will be proposed in this paper appears to overcome.

\subsection{Negative Exponentiated Model Weighting}

A common scheme for model weighting/averaging computes a nonnegative predictive score $s_{\mathcal{M}_k}$ for each model $\mathcal{M}_k$ and converts scores to probabilities via \emph{negative exponentiation}:

\begin{equation}
w_k=\frac{e^{-s_{\mathcal{M}_k}}}{\sum_{i=1}^K e^{-s_{\mathcal{M}_i}}}. \label{negative_exponentiation}
\end{equation}

Bayesian model averaging fits this template when $s_{\mathcal{M}_k}$ is the negative log marginal likelihood. Although alternatives exist (e.g., \cite{vassend2022}), negative exponentiation is standard and enjoys desirable theoretical properties in learning theory \citep{Vovk1990, lugosi2006, Hoeven2018}. Predictive scores can be defined in many ways. A prominent approach uses estimated optimism as in \ref{optimism}. Let $\hat{op}_k$ denote the estimated (expected) optimism of model $k$. If the goal is model selection, one chooses

\begin{equation}
\min_k \;\sum_i -\log p^p_k(y_i)+\hat{op}_k. \label{optimism_selection}
\end{equation}

In parametric settings, AIC \citep{Akaike1973, akaike1974, Sugiura1978, Hurvich1989, Burnham1998} approximates optimism by the parameter count; more generally, resampling or data splitting (e.g., cross-validation) is used. Optimism-based \emph{weighting} differs from selection by applying \eqref{negative_exponentiation} to the criterion instead of minimizing it directly. As shown in Section 4 of the Appendix, all such negative-exponentiated weightings implicitly solve

\begin{equation}
\min_{w^p\in S^K}\;\sum_k w_k^{p}\log\frac{w_k^{p}}{w^{op}_k}
- \sum_i\sum_k w_k^{p}\log p^p_k(y_i), \label{AIC_weightings}
\end{equation}

Despite being widely used, negative exponentiation has a well-known drawback: as the sample size grows, the weights tend to concentrate on the single best model, even when a convex combination would predict better. This issue, noted for Bayesian model averaging by \citet{Minka2002} and detailed in Section 4 of the Appendix, arises for all forms of negative-exponentiated weighting.

\subsection{Model Stacking}

Another approach treats model averaging as an estimation or optimization problem. A standard method is stacking \citep{Stone1974, wolpert1992, breiman1996, Leblanc1996, Yuling2018}, which chooses weights via cross-validation to optimize predictive accuracy. Originally for averaging point forecasts, \cite{Yuling2018} use the logarithmic score to stack Bayesian predictive distributions; however, \citeapos{Yuling2018} method applies beyond the Bayesian setting. We call this general procedure ``stacking with the log score.''

Let $F_1,\ldots,F_m$ be $m$ equally sized folds and $-F_j$ their complements. Fit each model on $-F_j$ (e.g., ML or Bayesian inference) to obtain leave-$F_j$-out predictors $p_k^{-F_j}$. Then solve

\begin{equation}
\min_{w\in S^K}\;\sum_{F_j}\sum_{y_i\in F_j} -\log\!\Big(\sum_k p_k^{-F_j}(y_i)\,w_k\Big). \label{stacking}
\end{equation}

Unlike negative exponentiated weighting, stacking directly estimates weights for the optimal linear combination of predictions. When no single model clearly dominates, stacking often outperforms methods based on negative exponentiation, including Bayesian model averaging \citep{Clarke2003}.

In machine learning, stacking often uses a two-level setup: leave-$F_j$-out predictors form ``level-one'' learners whose predictions feed a ``level-two'' meta-learner (or super-learner) \citep{Laan2007}, such as e.g., logistic regression, gradient boosting, or random forests. Unlike previous averaging schemes, the meta-learner need not produce probabilistic model weights to yield combined predictions.

Various forms of stacking enjoy good asymptotic properties \citep{Clarke2017, Yao2021}. Nonetheless, as we show empirically (simulations and datasets in Section 4; additional results in Section 2 of the Appendix), stacking can struggle with relatively small datasets, which is an issue that persists even when regularization is incorporated (see Section 4).

\section{THEORETICAL PERSPECTIVES ON THE DIVERGENCE-BASED APPROACH TO MODEL AVERAGING} 

In the next two subsections, we provide two theoretical justifications of divergence-based model weighting. Section 4 of the appendix provides yet another motivation for the method, which also shows how the method may be regarded as a modification of Bayesian model averaging.

\subsection{Divergence-Based Model Weighting as an Empirical Approximation of the Ideal Objective Function}

Since our goal is to maximize predictive accuracy on future data, we would ideally choose posterior model weights that minimize the following expression, where the expectation is with respect to the true (but unknown) probability distribution that governs the actual distribution of future data $y$:

\begin{equation} \min_{w^p \in S^K} \E[-\log{\sum _k  w_k^pp_k^p(y)}] \label{ideal_target}\end{equation}

Because this objective depends on unknown quantities, we must rely on a tractable surrogate. A natural first attempt is to minimize the loss of the model average evaluated on the observed data:

\begin{equation} \min_{w^p \in S^K}  \sum_i -\log{\sum _k  w_k^pp_k^p(y_i)}\label{empirical_target1} \end{equation}

However, (\ref{empirical_target1}) is biased, since it uses the same data both to fit the individual models and to assess the performance of the weighted average. This bias favors models that are overly optimistic (as defined in \ref{optimism}). Given the relationship between $p_k^p(\tilde{y}_i)$ and $p_k^p(y_i)$ in  (\ref{optimism}), a natural idea is to ``discount'' each weight $w_k^p$ by a factor proportional to $e^{-\hat{op_k}}$, where $\hat{op_k}$ is an estimate of model $k$'s optimism. But because the posterior weights are probabilities and must sum to one, absolute optimism levels are not what matter---what matters is each model’s optimism relative to the others. To that end we form the vector of \emph{optimism-penalizing prior model weights} as defined in \ref{optimism_weights} to downweight overly optimistic models and upweight relatively conservative ones. To incorporate this into the objective, we add a penalty term, $F(w^p, w^{op})$ that penalizes posterior weights in proportion to how far they diverge from the prior weights:

\begin{equation} \min_{w^p \in S^K} cF(w^p, w^{op}) - \sum_i \log{\sum _k  w_k^pp_k^p(y_i)}\label{empirical_target2} \end{equation}

Here, $c$ is a constant that regulates the influence of the prior weights as compared to the data in determining the optimal posterior weights. For $F$, a natural and flexible choice is the class of $f$-divergences, which take the following form, for some convex function $f$:

\begin{equation} F(w^p, w^{op}) = \sum_k w_k^{p}f(\frac{w_k^{op}}{w_k^{p}}) \end{equation}

To rule out pathological counter-examples, we assume that the generator function $f$ is sufficiently well-behaved: in particular, we assume that it is real analytic. As we discuss in the appendix, this assumption is much stronger than necessary, but has the benefit of being simple to state. In any case, most $f$-divergences used in practice have this property.\footnote{Importantly, even if we did not impose this extra condition in Theorem \ref{CharacterizationTheorem}, the theorem would still rule out all \emph{standard} $f$-divergences aside from the KL divergence.} The key questions now are (i) which $f$-divergence we ought to use and (ii) what value to assign to the constant $c$. To make headway on these questions, let's consider the special case where one of the models under consideration is substantially better than the other models, so that the best way of averaging the model predictions is simply to assign all (or almost all) the posterior weight to the single best model. A natural boundary condition is that, in this special case (\ref{empirical_target2}) should agree with the standard optimism-based model-selection criterion, i.e., (\ref{optimism_selection}). After all, when ordinary model selection is optimal, our weighting method should agree with it. 

Or, to state the same boundary condition differently, let $V^K$ be the set of vertices of the simplex $S^K$, i.e., the set of probability vectors $w$ such that $w_k = 1$ for some $k$. Then if we optimize (\ref{parsimony_stacking}) over $V^K$ rather than $S^K$, so that we force the method to select a single method, then the method should agree with the standard optimism-based model selection method expressed in (\ref{optimism_selection}). In other words, we require:

\begin{align}
\begin{gathered}
\arg\min_{w^p \in V^K} c\sum_k w_k^{p}f(\frac{w_k^{op}}{w_k^{p}}) - \sum_i \log\sum _k  w_k^pp_k^p(y_i) \\
= \arg\min_k  \sum_i -\log p^p_k(y_i) + op_k 
\end{gathered}\label{boundary_conditiona}
\end{align}

Imposing the boundary condition in (\ref{boundary_conditiona}) is enough to uniquely determine both $c$ and $f$:

\begin{restatable}{theorem}{CharacterizationTheorem}\label{CharacterizationTheorem}
If we require that the boundary condition (\ref{boundary_conditiona}) be satisfied in all situations and that the generator $f$ be real analytic, then:
\begin{enumerate}
\item The $f$-divergence in (\ref{empirical_target2}) must be the KL divergence. 
\item The constant $c$ in (\ref{empirical_target2}) must equal $1$.
\item Consequently, (\ref{empirical_target2}) coincides with the divergence-based model weighting problem (\ref{parsimony_stacking}).
\end{enumerate}
\end{restatable}

The brief reason why (\ref{boundary_conditiona}) entails Theorem \ref{CharacterizationTheorem} is that (\ref{boundary_conditiona}) forces the two methods to impose the same ranking over models, and this in turn forces them to be equal up to a constant offset that does not matter for optimization purposes. A more careful proof is in the appendix. 

Theorem \ref{CharacterizationTheorem} gives a strong theoretical reason for setting $c = 1$ and for using the KL divergence. In Section 5 of the appendix, we consider what happens if we let $c$ deviate from 1 in a linear regression experiment. The conclusion from this experiment is that $c > 1$ is somewhat better at smaller sample sizes, consistent with our theoretical analysis in Section 3.2.2 below, but that this comes at the price of worse performance at larger sample sizes. We therefore think $c = 1$ is a good default option. In Section 5 of the appendix, we also investigate what happens if we replace the KL divergence with the Brier divergence,  $\sum_k(w^p_k - w^{op}_k)^2$. The upshot is that the KL divergence performs substantially better. Finally, we also consider replacing the optimism-penalizing prior (\ref{optimism_weights}) with the flat distribution $w_k = 1/K$. Again, the conclusion is that the flat distribution results in much worse predictions than the optimism-penalizing prior (\ref{optimism_weights}).

\subsection{Justifications based on inequality bounds}

Readers familiar with PAC-Bayesian theory will notice a resemblance between the divergence-based model weighting optimization problem (\ref{parsimony_stacking}) and standard PAC-Bayesian bounds. This similarity makes it natural to try to analyze divergence-based model weighting from a PAC-Bayesian point of view, especially because PAC-Bayesian justifications of (certain sorts of) Bayesian model averaging have already been provided by \cite{Germain2016} and \cite{McAllester1999}.

However, in attempting to construct a PAC Bayesian analysis, we face two problems: (i) standard PAC Bayesian inequalities bound the expected average of the log likelihoods of the model, $\E[-\sum_k w_k^p \log p^p_k(y)]$, but what we are interested in bounding is the expectation of log of the linear mixture of the models, $\E[-\log \sum_k w_k^p p^p_k(y)]$, i.e., (\ref{ideal_target}). In other words, the log operator of PAC-Bayesian bounds is in the ``wrong'' place. (ii) Divergence-based model weighting  (\ref{parsimony_stacking}) uses the same data to both fit the models and then evaluate predictive accuracy, whereas standard PAC-Bayesian bounds assume that the models under consideration are evaluated on new data. In what follows we give two different inequality bounds: the first bound assumes that overfitting is insignificant and shows how we can convert a standard PAC Bayesian bound into a bound on (\ref{ideal_target}) by making assumptions about the tail behavior of the log scores of the models. The second bound allows for substantial overfitting and makes no assumptions about tail behavior, but is loose if overfitting is not substantial. All proofs are in the appendix.

\subsubsection{The first inequality bound}

Standard PAC-Bayes bounds place the logarithm in the ``wrong'' location. Still, we can often turn those bounds into ones we care about by assuming mild tail behavior for the models’ log scores. The key quantity is the pointwise Jensen gap, which measures how far the log of the mixture is from the mixture of the logs:

\begin{equation}
J_{out}(w^p) = -\sum_k w^p_k \log p_k^p(\tilde{y}_i) + \log \sum_k w^p_k p_k^p(\tilde{y}_i) 
\end{equation}

If the individual log scores are well behaved, this gap will usually be well behaved too. For illustration, assume the individual log losses are sub-Gaussian: there exists $s > 0$ such that for all $\lambda \in R$:

\begin{equation} \log \E[e^{{\lambda}(-\log p^p_k(\tilde{y}) - \E[-\log p^p_k(y)])} \le \frac{\lambda^2 s^2}{2}   \label{sub_gaussian1}\end{equation}

Because the model space is finite, sub-Gaussian log losses imply a sub-Gaussian Jensen gap. We have:

\begin{restatable}{theorem}{PACLemma}\label{PACLemma}
If each of $-\log p_k^p(\tilde{y}_i)$ is sub-Gaussian, then there exists a variance $s'$, such that for every set of posterior weights $w^p$, the Jensen gap $\sum_i J_{out}(w^p)$ is also sub-Gaussian with variance at most $ns'$.
\end{restatable}

Theorem (\ref{PACLemma}) shows that $\sum_i J_{out}(w^p)$ is subgaussian for all $w^p$. We now make the additional uniformity assumption that $\sup_{w^p} \sum_i J_{out}(w^p)$ is also subgaussian for some variance $n s^{''}$. This is an extra assumption, but it is plausible in many common scenarios. Applying a subgausian bound then gives us a general way of turning a PAC-Bayesian bound into a bound on our quantity of interest. For example, combining this assumption with the result with Corollary 4 of \cite{Germain2016} yields the following:

\begin{restatable}{theorem}{PACTheoremTwo}\label{PACTheorem2}
Suppose each model $-\log p_k^p(\tilde{y}_i)$ is sub-Gaussian with variance parameter $s$, and let $n s^{''}$ be defined as above. Then for any prior distribution $w_k$ and any $\delta \in (0,1]$, with probability at least $1-\delta$ we have, simultaneously for all distributions $w_k^p$:
\begin{align*}
\begin{gathered}
n\E[-\log \sum_k w^p_k p_k^p(y)] \\
\le \sum_k w_k^{p}\log\frac{w_k^p}{w_k}
- \sum_i\log \sum_k w^p_k p_k^p(\tilde{y}_i) \\
+ n\frac{1}{2}s^2
+ \log\frac{1}{\delta}
+ \sqrt{2ns^{''}\log\frac{1}{\delta}}
\end{gathered}
\end{align*}
\end{restatable}

The values of $s$ and $s^{''}$ of course matter to the tightness of the bound, but for optimization purposes, they are irrelevant. If overfitting is insignificant, so that $op_k \approx 0$, we can justifiably substitute the in-sample error $\sum_i-\log \sum _k w_k^pp_k^p(y_i)$ for the (unknown) out-of-sample $\sum_i-\log \sum _k w_k^pp_k^p(\tilde{y}_i)$ and hence optimizing Theorem (\ref{PACTheorem2}) reduces to optimizing the divergence-based model weighting problem (\ref{parsimony_stacking}). Importantly, in this setting the prior weights $w$ can be chosen freely; they need not be the optimism-penalizing weights of (\ref{optimism_weights}), since the assumption of negligible overfitting removes the need for explicit optimism correction.

A natural concern with Theorem \ref{PACTheorem2} is that the sub-Gaussian assumption may seem overly restrictive. This assumption was adopted for clarity of exposition rather than necessity. Weaker tail conditions---such as sub-exponentiality of the log-likelihood terms---yield bounds of the same general form, albeit looser. In Section 5 of the appendix, we do a linear regression experiment with heavy-tailed models. The result is still that divergence-based model weighting performs better than the alternatives.

\subsubsection{The second inequality bound}

In this section we wish to provide a bound without imposing any assumptions on the tail behavior of the models. Our goal here is to provide a bound that is instructive when overfitting is substantial. To quantify the degree of overfitting of each model, define: 

\begin{equation}\label{overfitting_degree}
m_k = \frac{\sum_i -\log p_k^p(\tilde{y}_i)}{\sum_i -\log p_k^p(y_i)}
\end{equation}

The ratio $m_k$ is a measure of the extent to which model $k$ overfits on the data. Note that in general $m_k > 1$. Let $m = \min_k m_k$. In cases where all models under consideration overfit, $m$ will be much greater than 1. By applying Jensen's inequality to the left side of the inequality in Theorem 3 in \cite{Germain2016} and using (\ref{overfitting_degree}) to simplify the right side, we have:

\begin{restatable}{theorem}{PACTheoremThree}\label{PACTheorem3}
Let $m_k$ be defined as in (\ref{overfitting_degree}) and let $m = \min_k m_k$. Assume $\sum_i -\log p_k^p(y_i) > 0$ for each model $k$, and $m > 1$. For any $\delta \in (0,1]$, with probability at least $1-\delta$ we have, simultaneously for all distributions $w_k^p$:
\begin{align*}
\begin{gathered}
n\E[-\log \sum_k w^p_k p_k^p(y)] 
\le \frac{m}{m-1}\sum_k w_k^{p}\log\frac{w_k^p}{w^{op}_k}  \\
- \frac{1}{m-1}\sum_k w^p_k \log w^p_k \\
+ \log K - \frac{m}{m-1} \log \sum_k e^{-op_k}
- \log \delta + \phi_{term}.
\end{gathered}
\end{align*}
\end{restatable}

Here, $\phi_{term} = \log \frac{1}{K}\sum_k \E[e^{\sum_i \log p^p_k(\tilde{y}_i) - n\E[\log p^p_k(y)]}]$, and does not depend on $w^p$. When $m$ is large (i.e., the models substantially overfit), we get, approximately:

\begin{align*}
\begin{gathered}
n\E[-\log \sum_k w^p_k p_k^p(y)] \le \sum_k w_k^{p}\log\frac{w_k^p}{w^{op}_k} + C
\end{gathered}
\end{align*}

Where $C$ collects terms that do not depend on $w^{p}$. Hence, when overfitting is substantial, we can estimate the optimal weights for $n\E[-\log \sum_k w^p_k p_k^p(y)]$ by simply optimizing $\sum_k w_k^{p}\log\frac{w_k^p}{w^{op}_k}$. This helps explain why both divergence-based model weighting (\ref{parsimony_stacking}) and negative exponentiated model weighting perform well when the sample size is very small, since both (\ref{parsimony_stacking}) and (\ref{AIC_weightings}) include, and will tend to be dominated by, $\sum_k w_k^{p}\log\frac{w_k^p}{w^{op}_k}$ in the small-sample case.

\subsubsection{An asymptotic result}

A natural concern is that divergence-based model weighting might systematically overfit, since it relies on the same data both to fit the models and to construct the prior model weights. If this were the case, the procedure could fail to converge to the ideal objective (\ref{ideal_target}). The following result shows, however, that if the methods used to fit each of the models individually converge, then the the empirical objective (\ref{empirical_target2}) converges to the ideal objective (\ref{ideal_target}) as the sample size grows. The proof relies on the log-sum inequality and is given in the supplementary materials.

\begin{restatable}{theorem}{ConvergenceTheorem}\label{ConvergenceTheorem}
For $k = 1, 2, \ldots, K$, let $p_k^p$ be a predictive distribution that results from fitting model  $\mathcal{M}_k$ to data  $D = \{y_1, y_2, \ldots, y_n\}$ that are distributed i.i.d. according to some distribution $Q$. Let $\tilde{D} = \{\tilde{y_1}, \tilde{y_2}, \ldots, \tilde{y_n}\}$ also be i.i.d. from $Q$, let $op_k  = -\sum_i \log p^p_k(\tilde{y}_i) + \sum_i\log p^p_k(y_i)$ be the optimism of each model. Let $S^K$ be the set of all probability functions with $K$ components. Suppose each $p_k^p$  converges in probability to a distribution $p_k^{*}$ in the sense that $\frac{1}{n}\sum_i|[\log p_k^{*}(y_i) - \log p_k^{p}(y_i)]| \xrightarrow{P} 0$ and $\frac{1}{n}\sum_i|[\log p_k^{*}(\tilde{y}_i) - \log p_k^{p}(\tilde{y}_i)]| \xrightarrow{P} 0$. Suppose also that all relevant integrals are finite. Then the following holds for all $w^p \in S^K$: 
\begin{align*}
\begin{gathered}
\frac{1}{n}|\sum_k w^p_k\log w^p_k + \sum_k w_k^{p}op_k - \sum_i \log{\sum _k  w_k^pp_k^p(y_i)} \\
-\E[-\log{\sum _k  w_k^pp_k^p(y)}]|  \xrightarrow{P} 0
\end{gathered}
\end{align*}
\end{restatable}

This result shows that divergence-based model weighting---in contrast to negative exponentiated model weighting, but similarly to model stacking---has the nice asymptotic property of converging towards the ideal (\ref{ideal_target}).

\section{EXPERIMENTAL RESULTS}
The following two subsections apply divergence-based model averaging in simulations and on data sets. Section B of the appendix applies the method in a Bayesian context and compares it to Bayesian stacking and Pseudo-BMA+ (Bayesian negative exponentiated model weighting) \citep{Yuling2018}.

\subsection{Linear Regression Simulation Experiment}

To show how divergence-based model weighting works in practice, we consider a linear regression simulation experiment, where the true distribution is of the following form:

\begin{equation}
Y = \mathbf{X}\beta + \alpha + \epsilon
\end{equation}

Here, $\beta$ is a 20-dimensional vector, $\alpha$ is an intercept, and $\epsilon$ is a normally distributed error term. We consider two possible ways of generating a precise data-generating distribution. In the first, non-sparse setting, the true values of the parameters of the data-generating distribution are determined as follows: the components of $\beta$ are generated independently from $\mathcal{N}(\mu  = 0, sd =  0.5)$; $\alpha$ is generated from $\mathcal{N}(\mu = 0, sd = 2)$; and $\epsilon$ is generated from $\mathcal{N}(0, 5)$. In the second, sparse setting, the values of the parameters are set in the same way, except that half of the components of the $\beta$ vector are randomly set to 0. To generate the design matrix, we also consider two possible scenarios. In the first scenario, all the entries of the design matrix $\mathbf{X}$  are generated independently from  $\mathcal{N}(\mu = 0, sd = 1)$. In the second scenario, the columns are set to have pairwise correlations of 0.5. In total, then, we consider four different types of data-generating distribution. 

Next, we create a set of 10 predictive models, where each model is constructed randomly in the following way: first, we select (with uniform probability) a number $m$ between one and five; next, $m$ predictors are randomly selected from the design matrix and are joined with an intercept term to form a normal, linear regression model. Of course, in practice we would not want to randomly construct our models. However, the purpose of this section is to give a fair assessment of how divergence-based model weighting performs given an arbitrary model space.\footnote{See the R code in the supplementary material for the precise procedure used to generate the models.} 

We generate a training set with a sample size of $n$  (ranging from 10 to 200) and a test set with 200 data points. We use maximum likelihood estimation to fit each of the models on the training set and then combine the predictions from the 10 fitted models on the test set by using the following methods:  (1) divergence-based model weighting, (2) Negative exponentiated model weighting with 5-fold CV, (3) model stacking with the log score. 

Finally, we calculate the root mean squared error (RMSE) of each of the model averaging methods on the test set. To get a sense of how the various model averaging methods perform across a range of model spaces and given different data-generating distributions, the entire preceding procedure is repeated 1000 times with different randomly created data-generating distributions and model spaces, and the results averaged. The results are shown in Figure \ref{frequentist_linear_regression}.

\begin{figure}[h]
        \centering
        \includegraphics[width=\columnwidth]{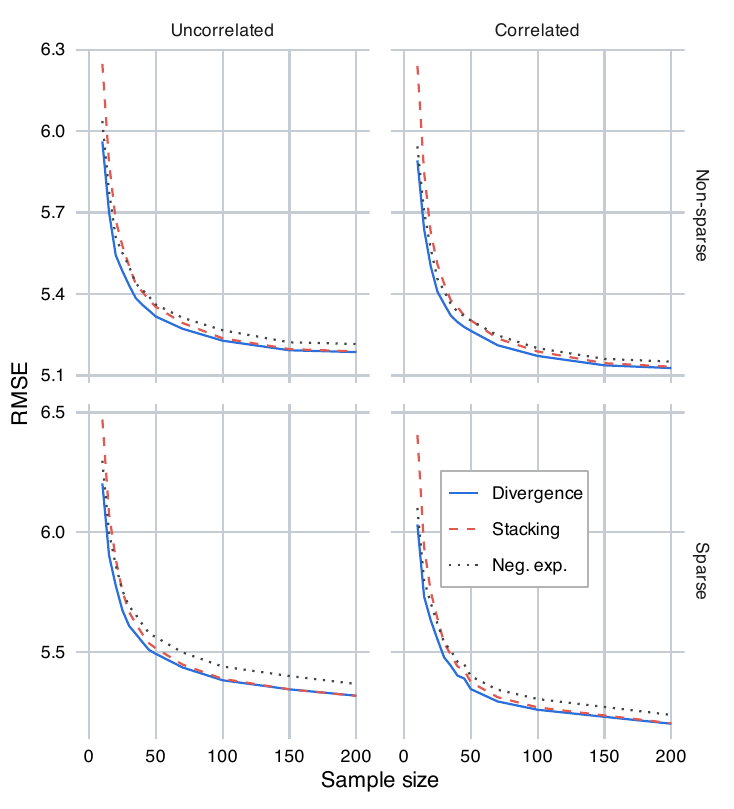}
    \caption{RMSE of model-weighting methods with various kinds of data-generating distribution. The standard error of each point is less than 0.04. \label{frequentist_linear_regression}}
\end{figure}

In this simulation, stacking and divergence-based weighting are asymptotically equal and better than Akaike style negative exponentiated weighting. However, for very small sample sizes, divergence-based weighting and Akaike weighting both perform substantially better than stacking. These results are in line with the claims we made in Section 2.

An important question concerns the stability of the model weights, especially because model optimism is estimated via cross-validation. Weight stability is also important for interpreting comparative predictive accuracy between models. To investigate this issue, we fix a ground-truth data-generating process and a set of models, and then repeat the simulation multiple times to examine how the estimated weights vary across runs. Figure \ref{weights_stability} shows, for each sample size, the standard deviation of each weight component (averaged across components) over 1000 runs.

\begin{figure}[h]
        \centering
       \includegraphics[width=\columnwidth]{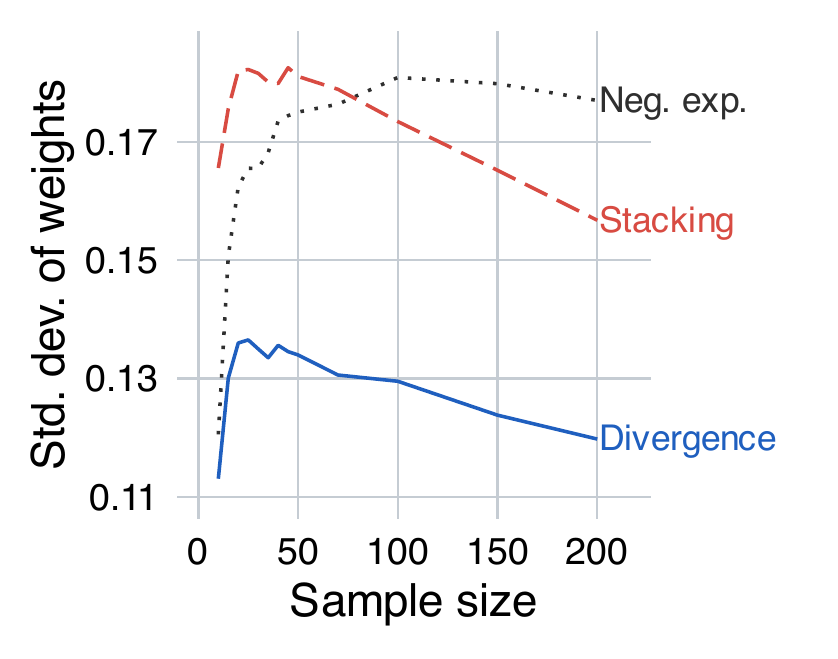}
    \caption{Standard deviation of model weights across 1000 simulation runs for different sample sizes. Lower values indicate more stable weights.\label{weights_stability}}
\end{figure}

Naturally, the precise values in Figure \ref{weights_stability} depend on the particular data-generating process and model space, but the qualitative pattern is representative.\footnote{Readers are encouraged to experiment with alternative parameter settings in the simulation.} As the figure shows, divergence-based model weighting has a systematically lower standard deviation at every sample size. Thus, in this experiment, divergence-based model weighting not only yields more accurate predictions but also produces more stable weights than stacking and negative exponentiated weighting.

\subsection{Using Divergence-Based Model Weighting to Average Predictions From Machine Learning Models}

\begin{table*}
\caption{Log scores of divergence-based model weighting (DW), stacking with log scores (LS), Akaike style negative exponentiated model weighting with log scores (NEW), stacking with a logistic meta-learner (GLM), stacking with an elastic net meta-learner (NET), and stacking with a gradient boosting meta-learner (GBM). The best score on each row is in bold (lower scores are better). The first two columns contain the number of observations ($n$) and predictors ($p$) in each data set. The last column shows the standard error across all repetitions of the mean difference between the best and second best score on each row.}\label{Table2}
\centering
\begin{tabular}{rrrrrrrrrr} 
  & $n$ & $p$ & DW & LS & NEW & GLM & NET & GBM & se  \\
 \hline
Cervical cancer &  65 & 19 & \textbf{0.270} & 0.287 & 0.272 & 0.705 & 0.298 & 0.371  & 0.0003  \\ 
Fertility & 99 & 9 &  \textbf{0.275} & 0.286 & 0.283 & 0.334 & 0.303 & 0.310  &  0.0008 \\ 
Heart failure & 299 & 11 &  \textbf{0.512} & \textbf{0.512} & 0.513 & 0.523 & 0.520 & 0.530 & 0.0004   \\ 
Heart disease 1 & 293 & 13 &  \textbf{0.387} & 0.389 & 0.394 & 0.401 & 0.400 & 0.411  & 0.0009  \\ 
Heart disease 2 & 303 & 13 &  \textbf{0.390} & 0.393 & 0.400 &  0.398 &  0.396 & 0.412 & 0.0008  \\ 
Diabetes & 520 & 16 & 0.076 & 0.077 & 0.075 & 0.066 & \textbf{0.065} &  0.073  & 0.0010 \\ 
Liver disease & 578 & 10 &  \textbf{0.494} & 0.498 &  0.503 & 0.504 & 0.504 & 0.510  & 0.0006  \\ 
Breast cancer &  682 & 9 &  \textbf{0.097} & 0.104 & 0.102 & 0.111 & 0.106 & 0.105 & 0.0009 \\ 
Australian credit & 689 & 14 & \textbf{0.322} & 0.323 & 0.330 & 0.324 & 0.323 & 0.328 & 0.0005 \\
German credit & 999 & 24 & \textbf{0.481} & 0.483 & 0.495 & 0.483 & 0.482 & 0.486 & 0.0003\\
Wine quality & 1599 & 11 & 0.418 & 0.415 & 0.418 & \textbf{0.409} & 0.411 & 0.418 & 0.0002  \\
Income prediction & 30161 & 14 & 0.326 & 0.321 & 0.345 & 0.334 & 0.335 & \textbf{0.312} & 0.0002 \\
 \hline
Mean & & &   \textbf{0.337} & 0.341 & 0.344 & 0.383 & 0.345 & 0.356 & \\ 
\hline
\end{tabular}
\end{table*}

We consider twelve data sets of varying sizes and dimensions from the UC Irvine Repository \citep{Dua2019}. For reasons of space, a complete description of the data sets is omitted, but a footnote gives the official name of each data set and a reference, if one is available.\footnote{In what follows, we refer to data sets in accordance with each data set's ``citation request'' in the UC Irvine Data Base. Not every data set has an associated publication. In descending order, the data sets in Table \ref{Table2} are: ``Cervical cancer behavior risk data set''  \citep{Sobar2016}. ``Fertility data set'' \citep{Gil2012}. ``Heart failure clinical records data set'' \citep{Chicco2020}. ``Hungarian heart disease data set'' (donated by Hungarian Institute of Cardiology. Budapest: Andras Janosi, M.D.). ``Cleveland heart disease data set'' (donated by V.A. Medical Center, Long Beach and Cleveland Clinic Foundation: Robert Detrano, M.D., Ph.D.). ``Early stage diabetes risk prediction dataset'' \citep{Islam2019}. 
``ILPD (Indian Liver Patient Dataset) Data Set''. ``Breast Cancer Wisconsin (Original) Data Set.'' ``Statlog (Australian Credit Approval) Data Set.'' ``Statlog (German Credit Data) Data Set.'' ``Wine Quality Data Set'' \citep{Cortez2009}. ``Adult/census income data set.''} 

Every data set is minimally cleaned and preprocessed: rows with missing entries are removed, factor predictors are converted to numbers, and finally all the predictors are standardized.  Each data set is divided into a training set consisting of 85$\%$ of the observations and a test set consisting of the remaining 15$\%$ of the observations. For data sets with $ n < 100$ observations, the results (and se) are averages of 1000 independent train/test splits; for data sets with $100  < n < 10000$, 500 train/test splits are used; finally, for data sets with $n > 10000$, 100 train/test splits are used.

We consider the following six models: logistic regression, elastic net regression, gradient boosting, support vector machine with a radial basis, random forest, and k-nearest neighbor. We use stratified 5-fold cross validation and the R machine learning interface Caret \citep{Kuhn2012} with the default settings (using nested cross validation) to fit the hyper-parameters of all models and to produce probabilistic predictions on the test set, and we also use 5-fold cross validation on the training data set to estimate model optimism. To average predictions from the models, we consider---as before---model stacking with the log score, divergence-based model weighting, and Akaike style negative exponentiated model weighting. But in addition, we also use the following standard ``meta-learners'' to form ``stacked'' model predictions: logistic regression, elastic net, gradient boosting. Note that all of these meta-learners incorporate regularization. The hyper-parameters of each meta-learner are tuned using 5-fold cross validation and the default settings of the CaretEnsemble \citep{Mayer2016} package.

Table \ref{Table2} gives the results. Divergence-based model weighting has the best log score on nine of the twelve data sets and has the best overall mean score. Moreover, in each pairwise comparison, divergence-based model weighting does better on at least ten of the twelve data sets. Note also that, in the four data sets that have the smallest size, negative exponentiated model weighting (NEW) performs better than all the methods that use various forms of stacking. However, divergence-based model weighting performs even better. And, by contrast with negative exponentiated model weighting, divergence-based model weighting also does well on the larger data sets. 

As seen in Table \ref{Table2}, divergence-based model weighting is not always the best-performing method. What might explain this? While we do not have a complete account, our results suggest a tentative explanation. Divergence-based model weighting assigns a separate weight to each model and is therefore linear in the model predictions, similarly to negative exponentiated model weighting and stacking with the log score (including Bayesian stacking). This linearity is advantageous when the sample size is small and is also important for interpretative purposes (especially model comparison). Moreover, Theorem 3.5 shows that divergence-based model weighting is an asymptotically optimal linear way of combining model predictions. However, for certain kinds of data, non-linear combinations can still be predictively superior. This may explain the rows of Table \ref{Table2} in which divergence-based weighting is outperformed by alternative (non-linear) methods.

\section{CONCLUSION}

Divergence-based model weighting is a theoretically well-motivated way of calculating model weights that performs well both in simulations and on data. The evidence we have presented suggests that the method performs well compared to the alternatives when sample sizes are small, while also retaining good performance with larger sample sizes. The method can be extended and possibly improved upon in several directions. First, it is possible to reformulate (\ref{parsimony_stacking}) in such a way that the model weights depend on the inputs, similar to how \cite{Sill2009} and \cite{Yao2021} extend stacking to make the stacking weights input-dependent. Second, it might be possible to find prior model weights that would be better than the optimism-penalizing weights we have proposed in this paper. We leave this as work for the future.

\section*{Acknowledgments}

Writing this paper took a tremendous amount of time and effort. I am greatly indebted to several reviewers who made the paper much stronger than it otherwise would have been. Work on this project has been supported by funding from the European Research Council (ERC) under the European Union’s Horizon Europe research and innovation programme (Grant agreement No. 101164097).

\section*{Checklist}

\begin{enumerate}

  \item For all models and algorithms presented, check if you include:
  \begin{enumerate}
    \item A clear description of the mathematical setting, assumptions, algorithm, and/or model. [Yes]
    \item An analysis of the properties and complexity (time, space, sample size) of any algorithm. [Not Applicable] The algorithm involves solving a relatively simple non-linear convex programming problem, which can be done almost instantly using the R package Rsolnp, for example. 
    \item (Optional) Anonymized source code, with specification of all dependencies, including external libraries. [Yes] This is in the supplementary materials.
  \end{enumerate}

  \item For any theoretical claim, check if you include:
  \begin{enumerate}
    \item Statements of the full set of assumptions of all theoretical results. [Yes]
    \item Complete proofs of all theoretical results. [Yes] This is provided in the appendix.
    \item Clear explanations of any assumptions. [Yes] In some cases, more complete explanations are provided in the appendix.    
  \end{enumerate}

  \item For all figures and tables that present empirical results, check if you include:
  \begin{enumerate}
    \item The code, data, and instructions needed to reproduce the main experimental results (either in the supplemental material or as a URL). [Yes]
    \item All the training details (e.g., data splits, hyperparameters, how they were chosen). [Yes] Some of this information has been relegated to the supplementary material.
    \item A clear definition of the specific measure or statistics and error bars (e.g., with respect to the random seed after running experiments multiple times). [Yes] 
    \item A description of the computing infrastructure used. (e.g., type of GPUs, internal cluster, or cloud provider).  [Not Applicable] The algorithm described in this paper can be run on an ordinary laptop from the last 10 years.

  \end{enumerate}

  \item If you are using existing assets (e.g., code, data, models) or curating/releasing new assets, check if you include:
  \begin{enumerate}
    \item Citations of the creator If your work uses existing assets. [Yes]
    \item The license information of the assets, if applicable. [Yes]
    \item New assets either in the supplemental material or as a URL, if applicable. [Not Applicable]
    \item Information about consent from data providers/curators. [Not Applicable] The data sets are from the UC Irvine Machine Learning Repository.
    \item Discussion of sensible content if applicable, e.g., personally identifiable information or offensive content. [Not Applicable]
  \end{enumerate}

  \item If you used crowdsourcing or conducted research with human subjects, check if you include:
  \begin{enumerate}
    \item The full text of instructions given to participants and screenshots. [Not Applicable]
    \item Descriptions of potential participant risks, with links to Institutional Review Board (IRB) approvals if applicable. [Not Applicable]
    \item The estimated hourly wage paid to participants and the total amount spent on participant compensation. [Not Applicable]
  \end{enumerate}

\end{enumerate}

\clearpage
\appendix

\onecolumn

\section{IMPLEMENTING DIVERGENCE-BASED MODEL WEIGHTING AND STACKING WITH THE LOG SCORE IN R}

To calculate the divergence-based model weights for a given set of models, we first need to calculate the pointwise predictive probabilities, $p_k^p(y_i)$, for each of the models and all data points, and gather the results in an $n$ by $K$ matrix. How to do this depends on which models we are fitting as well as what package we use to fit the models. For the sake of concreteness, our running example in this document will be Bayesian inference with a linear model, and we will assume that the models are fit to data using Stan  \citep{Stan2019}. In a Bayesian context, the pointwise predictive probability function $p_k^p(y_i)$ assumes the form of a pointwise posterior predictive probability (density) function \citep{Gelman2014}, $\int p_k(y_i | \theta)p(\theta|y)d\theta$, i.e., the posterior predictive probability (density) of each data point $y_i$, after the model has been fit to all the data. Since this latter expression is rather complicated, we will continue to use the generic notation $p_k^p(y_i)$, however. 

To calculate the pointwise posterior predictive probabilities for the Bayesian linear regression model using Stan, we include the following code in the ``generated quantities'' part of the Stan model code:

\begin{lstlisting}
  vector[N] log_lik;
  for (n in 1:N) log_lik[n] = normal_lpdf(y[n] | X[n, ] * beta + alpha, sigma);
\end{lstlisting}

We then extract the log-likelihoods from the posterior sample, exponentiate, and take the column averages to calculate the posterior predictive probability densities. (See \url{https://mc-stan.org/loo/reference/extract_log_lik.html} for more detailed information.) 

In addition to a matrix of pointwise posterior predictive probabilities, we also need to estimate the optimism of each model. There are several Bayesian estimates of model optimism,  e.g. the penalty term in the Widely Applicable Information Criterion \citep{Watanabe2013} or in the Deviance Information Criterion \citep{Spiegelhalter2002}, or the approximate Bayesian leave-one-out cross validation (LOO) estimate in \cite{vehtari2017}. In practice, our experience is that the estimates tend to be quite close to each other. In the next section, we use the LOO-based estimate, which \cite{vehtari2017} argue is most accurate. This estimate can be calculated easily using the LOO package \citep{LOO2020}. 

Once we have a vector of model optimism estimates (the unnormalized ``prior'') and a matrix of pointwise posterior predictive probabilities (which we call ``pointwise'' in the following), the divergence-based model weights can be calculated in R using the package RSolnp \citep{Rsolnp2015} with the following function:

\begin{lstlisting}
divergence_weights <- function(pointwise, prior){
  num_of_models <- ncol(pointwise)
  par_start <- rep(1, num_of_models)
  par_start <- par_start/sum(par_start)
  lower_bound <- rep(0, num_of_models) 
  upper_bound <- rep(1, num_of_models)
  equal <- function(x){ #Sum to 1 constraint
    sum(x)
  }
  optimizing_fn <- function(x){
    predictions <- pointwise%*%x
    score  <- -sum(log(predictions)) + 
      (sum(x*log(x)) + sum(x*prior))
    return(score)
  }
  weights <- solnp(pars = par_start, fun = optimizing_fn, 
                   eqfun = equal, eqB = 1, LB = lower_bound,
                   UB = upper_bound)$par
  return(weights)
}
\end{lstlisting}

In the next section, we compare the Bayesian implementation of divergence-based model weighting to Bayesian stacking and pseudobma+. The latter two methods are introduced in \cite{Yuling2018}, who demonstrate that both methods work better than a wide range of Bayesian model averaging methods, including standard Bayesian model averaging. Bayesian stacking is a Bayesian implementation of stacking with the log score, whereas pseudobma+ is a version of negative exponentiated model weighting that uses cross validation and bootstrapping to calculate model scores. Both methods rely on Pareto-smoothed importance sampling to approximate Bayesian leave-one-out cross validation. We refer the reader to \cite{Yuling2018} for further details.

To calculate pseudobma+ weights we use the aforementioned LOO package  \citep{LOO2020}. The LOO package can also be used to calculate stacking weights, but it tends to run into convergence issues, including in the simulations described in the next section.\footnote{Yuling Yao (personal communication) recommends using the function described in \cite{Yao2021}, which uses Stan to perform the optimization. We have tried this and confirm that it gives the same results as our function.} We therefore use the following function instead, where ``pointwise'' in this case refers to the leave-$k$-out matrix of model predictions (or the LOO approximation of this matrix):

\begin{lstlisting}
stacking_weights <- function(pointwise){
  num_of_models <- ncol(pointwise)
  par_start <- rep(1, num_of_models)
  par_start <- par_start/sum(par_start)
  lower_bound <- rep(0, num_of_models)
  upper_bound <- rep(1, num_of_models)
  equal <- function(x){
    sum(x)
  }
  optimizing_fn <- function(x){
    predictions <- pointwise%*%x
    score  <- -sum(log(predictions))
    return(score)
  }
  weights <- solnp(pars = par_start, fun = optimizing_fn, 
                   eqfun = equal, eqB = 1, LB = lower_bound,
                   UB = upper_bound)$par
  return(weights)
} 
\end{lstlisting}

\section{BAYESIAN SIMULATION EXPERIMENT}\label{Bayesian_regression}

This section replicates the linear subset regression simulations in \cite{Yuling2018}, who in turn adapt the example from \cite{breiman1996}. The set-up is as follows. First, covariates $X_1, X_2, \ldots, X_{15}$ are independently generated from $\mathcal{N}(5, 1)$. Next, $\alpha_j$ coefficients are calculated in the following manner, for $j = 1, 2, \ldots, 15$:

\begin{equation}
\alpha_j = 1_{|j-4| < h}(h - |j - 4|)^2 + 1_{|j-8| < h}(h - |j - 8|)^2 1_{|j-12| < h}(h - |j - 12|)^2
\end{equation}

Where $h$ is a parameter that determines how many of the coefficients are ``strong''. Following \cite{Yuling2018}, we put $h = 5$, which yields 15  ``weak'' coefficients. Next, $\beta_j$ coefficients are defined as follows: $\beta_j = \gamma\alpha_j$, where $\gamma$ is chosen such that the standard deviation of $\sum_j{\beta_jX_j}$ is equal to 2. We also generate errors $\epsilon$ independently from $\mathcal{N}(0, 1)$. Finally, we generate $Y$ values as follows:

\begin{equation}
Y = \beta_1X_1 + \ldots \beta_{15}X_{15} + \epsilon\label{generating_distribution}
\end{equation}

Again following \cite{Yuling2018}, we consider two different model spaces. The first model space consists of all models of the form $\mathcal{M}_k: Y \sim \mathcal{N}(\beta_kX_k, \sigma^2)$ (i.e., each model contains a single predictor) and the second model space consists of all models of the form $\mathcal{M}_k: Y \sim \mathcal{N}(\sum_k\beta_jX_j, \sigma^2)$ (the models are increasingly larger subsets of the set of all predictors).  For each model in both model spaces we use the following priors: $\beta_k \sim \mathcal{N}(0, 10)$ and $\sigma \sim$ Gamma$(0.1, 0.1)$. Note that neither of these model spaces (particularly not the first one) would realistically be used in practice. However, the two model spaces are still useful for the purposes of evaluating model averaging techniques since they represent two extremes: the case where all the models under consideration are radically misspecified, and the case where one of the models is true.

We generate a training data set from the true data distribution \ref{generating_distribution} consisting of $n$ data points, where $n$ ranges from 5 to 200, and we generate a test data set consisting of 200 data points and calculate the mean log predictive density of each model averaging method. We repeat this simulation 100 times and average the results to estimate the expected mean log predictive densities.

\begin{figure}[h]
        \centering
        \includegraphics[width=0.75\textwidth]{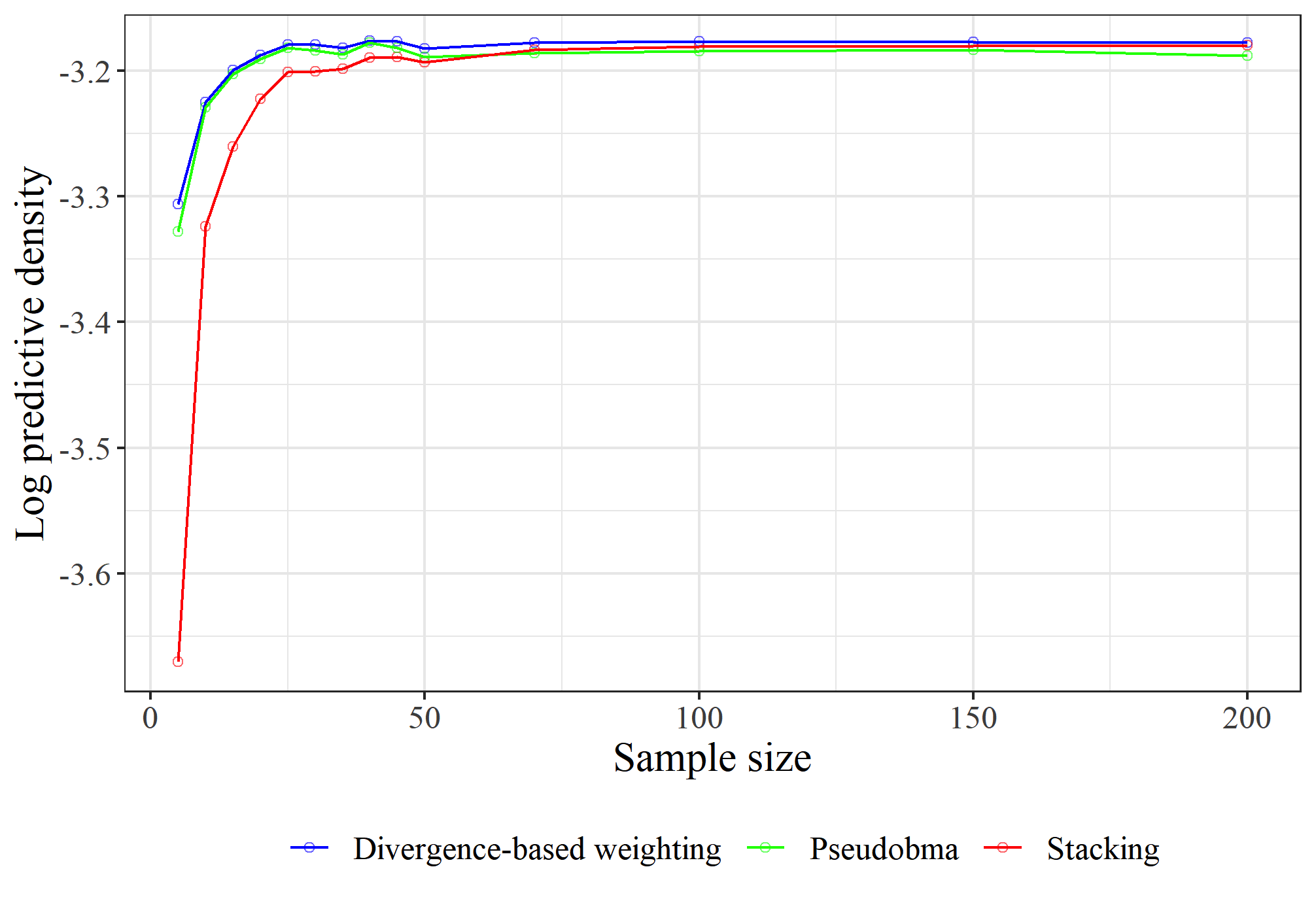} 
    \caption{First simulation: each model consists of a single predictor.\label{stan_regression_simulation_1}}
\end{figure}

\begin{figure}[h]
    \centering
        \includegraphics[width=0.75\textwidth]{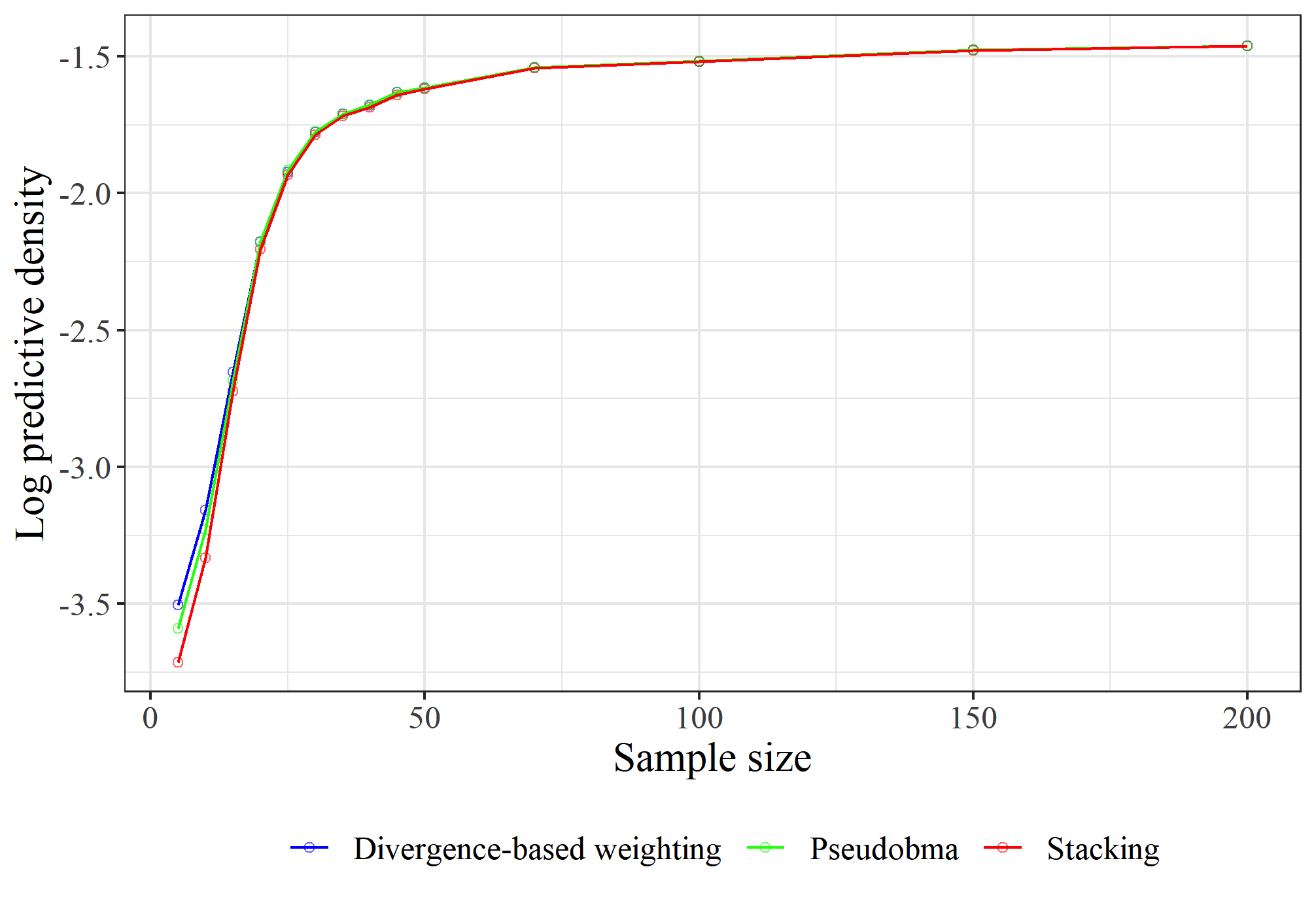} 
        \centering
    \caption{Second simulation: the models are increasingly larger subsets of the set of all predictors.\label{stan_regression_simulation_2}}
\end{figure}

Figure \ref{stan_regression_simulation_1} compares divergence weighting, Bayesian stacking, and pseudobma+ on the first model space. Pseudobma+ is better than Bayesian stacking for sample sizes smaller than 100, but Bayesian stacking is asymptotically better than pseudobma+.\footnote{Our results here differ somewhat from those reported in \cite{Yuling2018}. Comparing Figure 3 in \cite{Yuling2018} with our Figure \ref{stan_regression_simulation_1}, it is clear that the graphs for stacking are very similar, but pseudobma+ performs much better in our simulation than in \citeapos{Yuling2018}. We have not been able to determine the cause of this discrepancy.} Divergence-based model weighting  dominates both pseudobma+ and Bayesian stacking for small sample sizes in this experiment, and asymptotically behaves like Bayesian stacking. 

Figure \ref{stan_regression_simulation_2} compares the model weighting methods on the second model space. In this case, all three methods are about equally accurate for sample sizes larger than 25 and they all assign a probability of approximately 1 to the true model for large sample sizes. However, for very small sample sizes, divergence-based weighting is slightly more accurate than the alternatives.

\section{PROOFS OF THEOREMS}\label{proofs}

\subsection{Proof of Theorem 3.1 (the characterization theorem) from the main document}

For convenience, let us restate the boundary condition and the theorem. The boundary condition is:

\begin{align}\label{boundary_condition}
\begin{gathered}
\arg\min_{w^p \in V^K} c\sum_k w_k^{p}f(\frac{w_k^{op}}{w_k^{p}}) - \sum_i \log\sum _k  w_k^pp_k^p(y_i) \\
= \arg\min_k  \sum_i -\log p^p_k(y_i) + op_k 
\end{gathered}
\end{align}

Where $\sum_k w_k^{p}f(\frac{w_k^{op}}{w_k^{p}})$ is an $f$-divergence with a real analytic generator $f$. I.e., $f$ is a real analytic convex function such that $f(1) = 0$ and $\lim_{w \to 0} wf(\frac{1}{w}) = 0$.

And here is the theorem:

\CharacterizationTheorem*

\begin{proof}

If  we have $w_k = 1$ for some $k$, then $c\sum_k w_k^{p}f(\frac{w_k^{op}}{w_k^{p}}) - \sum_i \log\sum _k  w_k^pp_k^p(y_i) $ reduces to $cf(w_k^{op}) - \sum_i \log p_k^p(y_i)$. Hence, putting $S_k = -\sum_i \log p_k^p(y_i)$, the boundary condition (\ref{boundary_condition}) reduces to:

\begin{align}\label{boundary_condition2}
\begin{gathered}
\arg\min_k cf(w_k^{op}) + S_k \\
= \arg\min_k  op_k + S_k
\end{gathered}
\end{align}

By the definition of $w_k^{op}$, we have: $op_k = -\log w_k^{op} + Z$, where $Z = -\log \sum_j e^{-op_j}$ does not depend on $k$ and is therefore irrelevant in the optimization. Hence, we have:

\begin{align}\label{boundary_condition3}
\begin{gathered}
\arg\min_k cf(w_k^{op}) + S_k \\
= \arg\min_k  -\log w_k^{op} + S_k
\end{gathered}
\end{align} 

The fact that this holds in all situations means it holds for all ${w}^{op} \in S^K$ and all $S \in \mathbb{R}^K$. Now suppose there exist ${w'}^{op} \in S^K$ and all ${S'} \in \mathbb{R}^K$ such that there are indices $m$ and $n$ for which $cf({w'_{m}}^{op}) + {S'}_{m} \ge cf({w'_{n}}^{op}) + {S'}_{n}$ but $-\log {w'_{m}}^{op} + {S'}_{m} < -\log {w'_{n}}^{op} + {S'}_{n}$. Then, since we are free to vary $S$, we can create a new ${S''}$ with the same $m$ and $n$ indices as ${S'}$, but with all other entries of ${S''}$ very large, such that 
\[
m = \arg\min_k \big(-\log {w''_{k}}^{op} + {S''}_{k}\big) 
\quad\text{and}\quad 
n = \arg\min_k \big(cf({w'_{k}}^{op}) + {S'}_{k}\big),
\]
which contradicts (\ref{boundary_condition3}). Hence, there cannot exist such indices $m$ and $n$, and we conclude that (\ref{boundary_condition3}) entails the following for all indices $m$ and $n$:

\begin{align}\label{boundary_condition4}
\begin{gathered}
cf(w_m^{op}) + S_m \ge cf(w_n^{op}) + S_n  \iff  -\log w_m^{op} + S_m \ge -\log w_n^{op} + S_n
\end{gathered}
\end{align} 

Because $T = S_n - S_m$ can range freely in $\mathbb{R}$, \ref{boundary_condition4} entails that the following holds for all $T \in \mathbb{R}$:

\begin{align}\label{boundary_condition5}
\begin{gathered}
cf(w_m^{op}) - cf(w_n^{op}) \ge T  \iff  \log w_n^{op}  - \log w_m^{op}  \ge T
\end{gathered}
\end{align} 

But this entails that $cf(w_m^{op}) - cf(w_n^{op}) = \log w_n^{op}  - \log w_m^{op}$. Since it holds for all $w^{op}$, it must, in particular, hold when $w_n^{op} = 1$, and thus we get (since $f(1) = 0$ for f divergences:

\begin{align}\label{boundary_condition6}
\begin{gathered}
cf(w_m^{op}) = -  \log w_m^{op} 
\end{gathered}
\end{align} 

Since this holds for all possible values of $w_m^{op}$, and $w_m^{op}$ can be any number in the interval $(0, 1)$, $cf(x)$ is identical to $\log x$ in the interval $(0, 1)$. We now use the fact that $f$ is real analytic to conclude that $cf(x)$ is identical to $\log x$ for all $x$.\footnote{But note that what we technically need is the much weaker condition that identity on $(0, 1)$ entails identity on the whole domain.}  We can plug this into the boundary condition (\ref{boundary_condition}) and get:

\begin{align}\label{boundary_condition_final}
\begin{gathered}
\arg\min_{w^p \in V^K} \sum_k w_k^{p}\log\frac{w_k^{p}}{w_k^{op}} - \sum_i \log\sum _k  w_k^pp_k^p(y_i) \\
= \arg\min_k  \sum_i -\log p^p_k(y_i) + op_k 
\end{gathered}
\end{align}

And this establishes what we intended to show, i.e., $c = 1$ and the f-divergence is the KL divergence.

\end{proof}

\subsection{Proof of Theorem 3.2 and Theorem 3.3 from the main document}

let us first define the key quantities. First, we define the pointwise ``out-of-sample'' Jensen gap as follows:

\begin{equation}
J_{out}(w^p) = -\sum_k w^p_k \log p_k^p(\tilde{y}_i) + \log \sum_k w^p_k p_k^p(\tilde{y}_i) 
\end{equation}

Next, we assume that the individual log losses are sub-Gaussian: there exists $s > 0$ such that for all $\lambda \in R$:

\begin{equation} \log \E[e^{{\lambda}(-\log p^p_k(\tilde{y}) - \E[-\log p^p_k(y)])} \le \frac{\lambda^2 s^2}{2} \label{sub_gaussian1}\end{equation}

We will now prove Lemma 3.2 from the main document:

\PACLemma*

\begin{proof}

First, to simplify the notation, let:

\begin{equation}
X_k = -\log p_k^p(\tilde{y})
\end{equation}

\begin{equation}
\mu_k = \E[-\log p_k^p(\tilde{y})]
\end{equation}

Then the sub-Gaussian assumption says that, for each $k$ and all $\lambda$:

\begin{equation} \log \E[e^{\lambda(X_k -  \mu_k)}] \le \frac{\lambda^2 s^2}{2}\end{equation}\label{sub_gaussian2}

Let 

\begin{equation}
L = \sum_k w^p_k X_k
\end{equation}

and 

\begin{equation}
S = \log \sum_k w^p_k e^{X_k}
\end{equation}

We will first prove that each of $L$ and $S$ is sub-Gaussian. First, since $e^x$ is convex, Jensen's inequality implies: 

\begin{equation} \E[e^{\lambda (L - \E[L])}]  = \E[e^{\sum_k w^p_k\lambda(X_k - \mu_k})] \le \sum_k w^p_k \E[e^{\lambda(X_k -  \mu_k)}] \le  \frac{\lambda^2s^2}{2} \label{sub_gaussian4}\end{equation}

Thus, $L$ is sub-Gaussian with variance $s$. As for $S$, we have that:

\begin{equation}\label{bound}
\max_ k X_k  \le S \le \max_k X_k + \log K
\end{equation}

Hence, $S$ is trapped on both sides by $\max_k -X_k$ and a constant shift of $\max_k -X_k$. It follows that $S$ is sub-Gaussian if $\max_k -X_k$ is sub-Gaussian, which it is. Indeed, we have (see Proposition 2.7.6 of \cite{vershynin2025}): 

\begin{equation}
||\max_k X_k - \E[\max_k X_k]||_{\phi_2} \le s\sqrt{\log K}
\end{equation}

More carefully, (\ref{bound}) implies:

\begin{equation}
 || (S - \max_k X_k) - E[(S - \max_k X_k)]||_{\phi_2}   \le \log K
\end{equation}

And therefore, by the triangle inequality:

\begin{equation}
 ||S - E[S]||_{\phi_2}   \le ||\max_k X_k -  \E[\max_k X_k]||_{\phi_2} + \log K = s\sqrt{\log K} + \log K
\end{equation}

Hence, $S - E[S]$ is sub-Gaussian with some variance $t$.

Finally, we consider the (per sample) Jensen gap $L + S$:

\begin{align*}
\begin{gathered} 
\log \E[e^{{\lambda}((L + S) - \E[(L + S)])}] = \E[e^{{\lambda}(L - \E[L])}e^{{\lambda}(S - \E[S])}] \\
\le \sqrt{\E[e^{{2\lambda}(L - \E[L])}e^{{2\lambda}(S - \E[S])}]}  \\
\le  0.5 \frac{(2\lambda)^2s^2}{2} + 0.5 \frac{(2\lambda)^2t^2}{2} \\
= 2s^2 + 2t^2
 \label{sub_gaussian5}
 \end{gathered}
\end{align*}

Therefore, the pointwise Jensen gap is sub-Gaussian with variance $s'' = 2s^2 + 2t^2$. Finally, summing $L + S$ over $i$ independent $\tilde{y}_i$ yields a sum of $n$ independent subgaussian variables with variance $s''$, and therefore it has a variance of $n s''$.

\end{proof}

We will now prove Theorem 3.3. We need the following slightly reworded version of corollary 4 from \cite{Germain2016} (see the reference for a complete proof):

\begin{theorem}\label{PACTheorem}
Suppose each model  $\log p_k^p(\tilde{y}_i)$ is sub-Gaussian with variance parameter $s$. Then for any prior distribution $w_k$ and for any $\delta \in (0, 1]$, with probability at least $1 - \delta$ we have for all distributions $w_k^p$ simultaneously: 
\begin{align*}
\begin{gathered}
n\E[-\sum_k w^p_k \log p_k^p(y)] \le -\sum_i \sum_k w_k^p \log p_k^p(\tilde{y}_i) \\ 
 + \sum_k w^p_k\log\frac{w_k^p}{w_k} + \log{\frac{1}{\delta}} + n\frac{1}{2}s^2
\end{gathered}
\end{align*}
\end{theorem}

We also need the following fact about subgaussian variables: 

\begin{lemma}\label{Hoeffding}
Suppose $\sup_{w^p} \sum_i J_{out}(w^p)$ is subgaussian with variance $n s^{''}$. Then, for all $\delta \in (0, 1]$, with probability at least $1 - \delta$, the following holds for all $w^p$ simultaneously:
\begin{equation}  n\E[\sum_i J_{out}(w^p)] \le \sum_i J_{out}(w^p)] + \sqrt{2ns^{''}\log\frac{1}{\delta}}\end{equation}
\end{lemma}

Here is the result we will prove:

\PACTheoremTwo*

\begin{proof}

Lemma \ref{Hoeffding} implies that, for any $\delta \in (0, 1]$, with probability at least $1 - \delta$, the following holds:

\begin{equation}
\begin{gathered} 
n\E[-\log \sum_k w^p_k p_k^p(y) + \sum_k w^p_k \log p_k^p(y)] \le -\sum_i\log \sum_k w^p_k p_k^p(\tilde{y}_i) + \sum_i\sum_k w^p_k \log p_k^p(\tilde{y}_i) + \sqrt{2ns^{''}\log\frac{1}{\delta}} \label{sub_gaussian5}\end{gathered}
\end{equation}

Rearranged, this implies: 

\begin{equation}
\begin{gathered} 
n\E[-\log \sum_k w^p_k p_k^p(y)] \le n\E[-\sum_k w^p_k \log p_k^p(y)] + \sum_i\sum_k w^p_k \log p_k^p(\tilde{y}_i) - \sum_i\log \sum_k w^p_k p_k^p(\tilde{y}_i) +  \sqrt{2ns^{'}\log\frac{1}{\delta}}
 \label{sub_gaussian6}\end{gathered}
\end{equation}

Theorem \ref{PACTheorem} further implies that for any prior distribution $w_k$ and for any $\delta \in (0, 1]$, with probability at least $1 - \delta$, the following holds for all distributions $w^p$ simultaneously:

\begin{equation}
\begin{gathered} 
n\E[-\sum_k w^p_k \log p_k^p(y)] + \sum_i\sum_k w^p_k \log p_k^p(\tilde{y}_i) \le  \sum_k w_k^{p}\log\frac{w_k^p}{w_k} + \log{\frac{1}{\delta}} + n\frac{1}{2}s^2
 \label{sub_gaussian7}\end{gathered}
\end{equation}

Putting \ref{sub_gaussian6} and \ref{sub_gaussian7} together yields that the following holds for any prior distribution $w_k$ and for any $\delta \in (0, 1]$, with probability at least $1 - \delta$ for all distributions $w_k^p$ simultaneously:

\begin{equation}
\begin{gathered} 
n\E[-\log \sum_k w^p_k p_k^p(y)] \le  \sum_k w_k^{p}\log\frac{w_k^p}{w_k} - \sum_i\log \sum_k w^p_k p_k^p(\tilde{y}_i) + n\frac{1}{2}s^2 + \log{\frac{1}{\delta}} +  \sqrt{2ns^{''}\log\frac{1}{\delta}}
 \label{sub_gaussian5}\end{gathered}
\end{equation}

\end{proof}

\subsection{Proof of Theorem 3.4 from the main document}

We will use the following theorem \citep{Germain2016}: 

\begin{theorem}\label{PACTheorem_general}
For any prior distribution $w_k$ and for any $\delta \in (0, 1]$ and $c > 0$, with probability at least $1 - \delta$ we have for all distributions $w_k^p$ simultaneously: 
\begin{align*}
\begin{gathered}
n\E[-\sum_k w^p_k \log p_k^p(y)] \\
\le  - \sum_i\sum_k w^p_k \log p_k^p(\tilde{y}_i)  + \frac{1}{c}\sum_k w_k^{p}\log\frac{w_k^p}{w_k} \\
-\frac{1}{c}\log \delta + \phi_{term}
\end{gathered}
\end{align*}
\end{theorem}

Here, $\phi_{term} = \log \frac{1}{K}\sum_k \E[e^{\sum_i \log p^p_k(\tilde{y}_i) - n\E[\log p^p_k(y)]}]$. For simplicity, let $T = -\frac{1}{c}\log \delta + \phi_{term}$

Define the following quantities: 

\begin{equation}
J_{IN} = -\sum_i\sum_k w^p_k \log p_k^p(y_i) + \sum_i\log \sum_k w^p_k p_k^p(y_i) 
\end{equation}

\begin{equation}
m_k = \frac{\sum_i -\log p_k^p(\tilde{y}_i)}{\sum_i -\log p_k^p(y_i)}
\end{equation}

$J_{IN}$ is the ``in-sample'' Jensen Gap calculated on the same data as  was used to fit the models. $m_k$ is a measure of the extent to which model $k$ overfits on the data. Note that in general $m_k > 1$. Let $m = \min_k m_k$. In cases where all models under consideration overfit, $m$ will tend to be much greater than 1. Using (\ref{PACTheorem_general}), we have (by applying Jensen's inequality on the first line and setting $c = 1$):

\begin{align}
\begin{gathered}
n\E[-\log \sum_k w^p_k p_k^p(y)]  \le n\E[-\sum_k w^p_k \log p_k^p(y)]  \\
\le - \sum_i\sum_k w^p_k \log  p_k^p(\tilde{y}_i)  + \sum_k w_k^{p}\log\frac{w_k^p}{w_k} + T \\
= - \sum_i\sum_k w^p_k \log p_k^p(y_i) + \sum_k w^p_k op_k + \sum_k w_k^{p}\log\frac{w_k^p}{w_k} + T \\
= - \sum_i\sum_k w^p_k \log p_k^p(y_i)  + \sum_k w_k^{p}\log\frac{w_k^p}{w^{op}_k} \\
- \sum_k w^p_k \log w_k - \log \sum_k e^{-op_k} + T \\
= - \sum_i\log\sum_k  w^p_k p_k^p(y_i) +  \sum_k w_k^{p}\log\frac{w_k^p}{w^{op}_k} \\
+ J_{IN} - \sum_k w^p_k \log w_k - \log \sum_k e^{-op_k} + T
\end{gathered}\label{Implication}
\end{align}

We also have:

\begin{align}
\begin{gathered}
J_{IN}  = -\sum_i\sum_k w^p_k \log p_k^p(\tilde{y}_i)  + \sum_i\log\sum_k w^p_k p_k^p(y_i)  \\
+ \sum w^p_k \log w^p_k - \sum_k w_k^{p}\log\frac{w_k^p}{w^{op}_k} + \log \sum_k e^{-op_k} \\
\ge -m\sum_i\sum_k w^p_k \log p_k^p(y_i)  + \sum_i\log\sum_k w^p_k p_k^p(y_i) \\
+ \sum w^p_k \log w^p_k - \sum_k w_k^{p}\log\frac{w_k^p}{w^{op}_k} + \log \sum_k e^{-op_k} \\
= mJ_{IN} - (m-1)\sum_i\log\sum_k w^p_k p_k^p(y_i) + \sum w^p_k \log w^p_k \\
 -\sum_k w_k^{p}\log\frac{w_k^p}{w^{op}_k} + \log \sum_k e^{-op_k}
\end{gathered}\label{Jensen_gap_inequality}
\end{align}

This implies: 

\begin{align}
\begin{gathered}
(m-1)J_{IN} \le (m-1)\sum_i\log\sum_k w^p_k p_k^p(y_i)  -\sum w^p_k \log w^p_k \\
+ \sum_k w_k^{p}\log\frac{w_k^p}{w^{op}_k} - \log \sum_k e^{-op_k}
\end{gathered}\label{Jensen_gap_inequality2}
\end{align}

Thus, combining (\ref{Implication}) and (\ref{Jensen_gap_inequality2}), we have:

\begin{align}
\begin{gathered}
n\E[-\log \sum_k w^p_k p_k^p(y)] \\
\le - \sum_i\log\sum_k  w^p_k p_k^p(y_i) +  \sum_k w_k^{p}\log\frac{w_k^p}{w^{op}_k} \\
+ J_{IN} - \sum_k w^p_k \log w_k - \log \sum_k e^{-op_k} + T\\
\le  \frac{m}{m-1}\sum_k w_k^{p}\log\frac{w_k^p}{w^{op}_k}  \\
- \frac{m}{m-1} \log \sum_k e^{-op_k} - \frac{1}{m-1}\sum w^p_k \log w^p_k \\
- \sum_k w^p_k \log w_k + T
\end{gathered}
\end{align}

When $m$ is large (i.e., the models overfit), we get, approximately (and setting $w_k = 1/K)$:

\begin{align*}
\begin{gathered}
n\E[-\log \sum_k w^p_k p_k^p(y)] \le \\
 \sum_k w_k^{p}\log\frac{w_k^p}{w^{op}_k} + \log K - \log \sum_k e^{-op_k} + T
\end{gathered}
\end{align*}

\subsection{Proof of Theorem 3.5 From the Main Document}

For convenience, we first restate Theorem 3.5 (here renumbered as Theorem \ref{ConvergenceTheorem}):

\ConvergenceTheorem*

\begin{proof} 

We need the sum-log inequality:

\begin{lemma}\label{log_sum_inequality}
Suppose $A_1, A_2, \ldots, A_K$ and $B_1, B_2, \ldots, B_K$ are sequences of positive numbers. Then:
\begin{equation*}
\log\frac{\sum_k A_k}{\sum_k B_k} \le \sum_k \frac{A_k}{\sum_j A_j}\log \frac{A_k}{B_k}
\end{equation*}
\end{lemma}

We can now prove Theorem \ref{ConvergenceTheorem}. 

We start by expanding $|\frac{1}{n} \sum_k w^p_k\log w^p_k + \frac{1}{n}\sum_k w_k^{p}op_k - \frac{1}{n}\sum_i \log{\sum _k  w_k^pp_k^p(y_i)} - \E[-\log{\sum _k  w_k^pp_k^p(y)}]|$ as follows:

\begin{equation}\label{expansion}
\begin{split} 
|\frac{1}{n} \sum_k w^p_k\log w^p_k + \frac{1}{n}\sum_k w_k^{p}op_k  - \frac{1}{n}\sum_i \log{\sum _k  w_k^pp_k^p(y_i)} - \E[-\log{\sum _k  w_k^pp_k^p(y)}]| \\
\le |\frac{1}{n} \sum_k w^p_k\log w^p_k| \\
+ |\frac{1}{n}\sum_k w_k^{p}op_k| \\
 + \frac{1}{n}|\sum_i\log\sum_k w_k^pp_k^{*}(y_i) -  \sum_i\log\sum_k w_k^pp_k^{*}(\tilde{y}_i) | \\
 + \frac{1}{n}|\sum_i\log\sum_k w_k^pp_k^{*}(\tilde{y}_i) -  \sum_i\log{\sum _k  w_k^pp_k^p(\tilde{y}_i)}| \\
 + \frac{1}{n}|\sum_i\log\sum_k w_k^pp_k^{p}(y_i) - \sum_i\log\sum_k w_k^pp_k^{*}(y_i) | \\
 + |\E[-\log{\sum _k  w_k^pp_k^p(y)}] - \frac{1}{n}\sum_i[-\log{\sum _k  w_k^pp_k^p(\tilde{y}_i)}|
\end{split}
\end{equation}

Every line after the inequality sign in \ref{expansion} converges (in probability) to 0 as $n \rightarrow \infty$.  First, $|\sum_k w^p_k\log w^p_k||$ is bounded for all $w_k^p$ by $|\log K|$, so $\frac{1}{n}|\sum_k w^p_k\log w^p_k|$ simply converges to 0. 

Second, we have: 

\begin{equation}
\begin{split} 
\frac{1}{n}|\sum_k w_k^{p}op_k| = \frac{1}{n}|\sum_k w^p_k(\sum_i \log p^p_k(\tilde{y}_i) - \sum_i\log p^p_k(y_i))| \\
\le \frac{1}{n}|\sum_i\sum_k w^p_k \log p^p_k(\tilde{y}_i) - \log p^p_k(y_i))| = \frac{1}{n}|\sum_i\sum_k w^p_k(\log p^p_k(\tilde{y}_i) - \log p_k^{*}(\tilde{y}_i) + \log p_k^{*}(\tilde{y}_i) \\
- \log p_k^{*}(y_i) + \log p_k^{*}(y_i)  - \log p^p_k(y_i))| \\
\le \frac{1}{n}|\sum_i\sum_kw^p_k(\log p^p_k(\tilde{y}_i) - \log p_k^{*}(\tilde{y}_i))| \\
+ \frac{1}{n}|\sum_i\sum_k w^p_k (\log p_k^{*}(\tilde{y}_i) - \log p_k^{*}(y_i))| \\
+ \frac{1}{n}|\sum_i\sum_k w^p_k(\log p_k^{*}(y_i)  - \log p^p_k(y_i)))|
\end{split}
\end{equation}

In the last three lines, the first and third converge (in probability) to 0 by the assumption of the theorem. The second term converges (in probability) to 0 by the Weak Law of Large Numbers. Hence, $\frac{1}{n}|\sum_k w_k^{p}op_k|$ converges (in probability) to 0.

Next, $ |\frac{1}{n}\sum_i\log\sum_k w_k^pp_k^{*}(y_i) -  \sum_i\log\sum_k w_k^pp_k^{*}(\tilde{y}_i) |$ converges (in probability) to 0 by the Weak Law of Large Numbers.

We next show that $\frac{1}{n}|\sum_i\log\sum_k w_k^pp_k^{*}(\tilde{y}_i) -  \sum_i\log{\sum _k  w_k^pp_k^p(\tilde{y}_i)}|$ converges (in probability) to 0 by using the log-sum probability. 

Let $A_{ik} = \frac{w_k^pp_k^{*}(\tilde{y}_i)}{\sum_j w_j^pp_j^{*}(\tilde{y}_i)}$. Note that $0 \le A_{ik} \le 1$ and $\sum_k A_{ik} = 1$. First, assume $\sum_i\log\sum_k w_k^pp_k^{*}(\tilde{y}_i) -  \sum_i\log{\sum _k  w_k^pp_k^p(\tilde{y}_i)} \ge 0$. By the log-sum inequality, we have:

\begin{equation}
\begin{split} 
|\frac{1}{n}\sum_i\log\sum_k w_k^pp_k^{*}(\tilde{y}_i) - \frac{1}{n}\sum_i\log\sum_k w_k^pp_k^{p}(\tilde{y}_i)| \\
 =  \frac{1}{n}|\sum_i\log\frac{\sum_k w_k^pp_k^{*}(\tilde{y}_i)}{\sum_k w_k^pp_k^{p}(\tilde{y}_i)}| \\ 
 \le \frac{1}{n}|\sum_k\sum_i A_{ik}\log\frac{w_k^pp_k^{*}(\tilde{y}_i)}{w_k^pp_k^{p}(\tilde{y}_i)}| \\
  \le \frac{1}{n}\sum_k\sum_i A_{ik}|\log\frac{p_k^{*}(\tilde{y}_i)}{p_k^{p}(\tilde{y}_i)}| \\
    \le \frac{1}{n}\sum_k\sum_i |\log p_k^{*}(\tilde{y}_i) - \log p_k^{p}(\tilde{y}_i)| \\
\end{split}
\end{equation}

The last line converges (in probability) to 0 by our theorem's assumption. If $\sum_i\log\sum_k w_k^pp_k^{*}(\tilde{y}_i) -  \sum_i\log{\sum _k  w_k^pp_k^p(\tilde{y}_i)} \le 0$, we can do the same argument with the roles of $p_k^{*}(\tilde{y}_i)$ and $p_k^{p}(\tilde{y}_i)$ reversed and arrive at the same conclusion.

By exactly the same argument, we can show that $\frac{1}{n}|\sum_i\log\sum_k w_k^pp_k^{p}(y_i) - \sum_i\log\sum_k w_k^pp_k^{*}(y_i)|$ converges (in probability) to 0. 

Finally, by the Weak Law of Large Numbers $|\E[-\log{\sum _k  w_k^pp_k^p(y)}] - \frac{1}{n}\sum_i[-\log{\sum _k  w_k^pp_k^p(\tilde{y}_i)}|$ converges (in probability) to 0.

Hence, we conclude that all the terms in (\ref{expansion}) converge (in probability) to 0.

\end{proof}

\section{A third justification of divergence-based model weighting: divergence-based model weighting as a modification of Bayesian model averaging}\label{third_justification}

\cite{bissiri2016} show that Bayesian updating may be regarded as the solution to a certain optimization problem, where---speaking loosely---the goal is to have a posterior distribution over the parameter values that achieves the dual goal of being close to the prior distribution (and therefore suitably skeptical and cautious) and also assigns a high posterior probability (density) to parameter values that are predictively accurate on the data.  More formally, let $p^p(\theta)$ and $p(\theta)$ be the posterior and prior distributions over $\theta$, respectively, and let $S$ be the set of all probability distributions over $\theta$. Then \cite{bissiri2016} formulate the following optimization problem, which they show is uniquely minimized by the Bayesian posterior distribution: 

\begin{equation} \min_{p^p \in S} \int p^{p}(\theta)\log\frac{p^p(\theta)}{p(\theta)}d\theta - \sum_i {\int p^p(\theta) \log p(y_i| \theta)d\theta}\label{bissiri} \end{equation}

In fact, \cite{bissiri2016} consider a modified version of \ref{bissiri}, where $\log p(y_i|\theta)$ is replaced with a generic loss function. They argue that this is a legitimate move from a decision theoretic point of view, since \ref{bissiri} may be regarded as an expected loss that consists of two separate pieces: the loss that results from having a posterior distribution that differs from the prior distribution (quantified in the first term of \ref{bissiri} as the KL divergence \citep{kullbackleibler1951} from the posterior to the prior); and the loss that results from assigning high posterior probabilities to parameter values that have a poor fit to the data (quantified in the second term of \ref{bissiri} as the (negative) log-likelihood function averaged over the posterior). If there is reason to think that the model is misspecified, then \cite{bissiri2016} argue that there is no reason why predictive accuracy should necessarily be measured in terms of the log-likelihood. They furthermore show that when the log-likelihood is replaced with a different loss function, Bayesian updating may no longer be optimal.

Here, we also seek to modify the optimization problem in (\ref{bissiri}), but in a different direction. First, instead of considering the problem of assigning an optimal posterior distribution over a parameter in a model, we are instead interested in the problem of assigning a set of optimal posterior weights $w_k^p$ over a set of statistical or machine learning models, where---using the notation in (1)---each model has been fit to all the data, thereby producing a predictive distribution $p_k^p(y_i)$. Hence, we first reformulate (\ref{bissiri}) as follows:

\begin{equation} \min_{w^p \in S^K} \sum_k w_k^{p}\log\frac{w_k^p}{w_k} - \sum_i \sum _k  w_k^p\log p_k^p(y_i)\label{bissiri_reformulated} \end{equation}

In (\ref{bissiri_reformulated}), posterior model weights are rewarded (in the second term) to the extent that they assign a high probability to models that are individually predictively accurate (as measured by the log score). In fact, different Akaike style model weighting methods may be regarded as providing solutions to optimization problems of the form in (\ref{bissiri_reformulated}). For example, let $w_k^{AIC} \propto e^{-(c_k + \frac{c_k(c_k+1)}{n - c_k - 1})}$ be the prior model weights and let $p_k^p(y_i)  = p_k(y_i|\hat{\theta})$, where $\hat{\theta}$ is the maximum likelihood of the parameter $\theta$ of model $k$. Plugging these terms into (\ref{bissiri_reformulated}) yields the following optimization problem:

\begin{equation} \min_{w^p \in S^K} \sum_k w_k^{p}\log\frac{w_k^p}{w_k^{AIC}} - \sum_i \sum _k  \log p_k(y_i|\hat{\theta})w_k^p \label{AIC_weighting} \end{equation}

The general results in \cite{bissiri2016} now imply that the posterior model weights that optimize (\ref{AIC_weighting}) are proportional to $\prod_i p_k(y_i|\hat{\theta})e^{-(c + \frac{c(c+1)}{n - c - 1})}$, which are the small sample corrected form of the well known Akaike model weights \citep{Akaike1979}. Thus, Akaike model weighting and negative exponentiated model weighting more generally have a theoretical foundation in the general divergence framework presented in \cite{bissiri2016}. That is, we may regard negative exponentiated model weighting as  a process that involves two steps: first, formulate a set of prior model weights that guard against overfitting (e.g., the aforementioned Akaike prior weights); second, choose the posterior model weights that optimize (\ref{bissiri_reformulated}). We may therefore call the optimization problem in (\ref{bissiri_reformulated}) the ``negative exponentiated optimization problem.'' 

Note, however, that in a model weighting and averaging context, the optimization problem in  (\ref{bissiri_reformulated}) does not really capture what we are after. Indeed, in a Bayesian context, \cite{Masegosa2020} argues that optimizing (\ref{bissiri}) will often result in suboptimal predictions, and since (\ref{bissiri_reformulated}) is a minor modification of (\ref{bissiri}) there is good reason to think the same will be true for (\ref{bissiri_reformulated}). The main problem, which was pointed out by \cite{Minka2002} in the context of Bayesian model averaging, is that negative exponentiated model weighting will tend to increasingly concentrate the probability on the single best model. This is clear from  (\ref{bissiri_reformulated}), because as the amount of data increases, $\sum_i \sum _k  w_k^p\log p_k^p(y_i)$ will increasingly be dominated by the probability model that has the best log score on the data. However, what we would like is to reward posterior weights that produce accurate \emph{combinations} of model predictions. Fortunately, this can be achieved in a simple way by replacing the second term of (\ref{bissiri_reformulated}) with  $\sum_{i=1}^n -\log \sum_{k = 1}^K w_k^p{p_k^p(y_i)}$, i.e., we simply move the position of the log operator. A similar move is recommended by \cite{Morningstar2022} in a Bayesian context. Moving the log operator yields an objective function that is more sensible---given our aims---than (\ref{bissiri_reformulated}), namely the divergence-based model weighting problem (1) from the main document. The divergence-based model weighting problem formally quantifies the dual goal of having a set of posterior model weights that are close to the prior weights (and therefore suitably skeptical) while also yielding a \emph{linear combination} of model predictions that have a good fit to the data. Thus, in this way, divergence-based model weighting arises as a natural modification of negative exponentiated model weighting. Indeed, for small sample sizes, the KL divergence term---which is shared by the divergence-based model weighting optimization problem (1) and the negative exponentiated model weighting problem (\ref{bissiri_reformulated})---will be relatively influential, and hence we might expect the divergence-based model weights and the negative exponentiated model weights to be fairly similar. However, for large sample sizes, the KL divergence term will play a negligible role, and the divergence-based model weighting optimization problem (1) will be more similar to the stacking optimization problem (5) from the main document.

Given the close connection between the divergence-based optimization problem (1) and the Bayesian optimization problem (\ref{bissiri}) and its modification  (\ref{bissiri_reformulated}), we think it is natural to interpret the divergence-based model weights in a broadly Bayesian way: $w_k^p$ reflects the (relative) posterior trust we place in model $\mathcal{M}_k$ for predictive purposes. As such, $w_k^p$ should not be regarded as an uncertain estimate of some underlying ``true'' model weight, in the same way that Bayesian posterior probabilities are not typically regarded as imperfect estimates of ``true'' probabilities.

\section{Robustness checks}

The aim of this section is to do various robustness checks. First, in Section 3.1, we argued that the KL-divergence is a well-motivated penalty term, but one might naturally wonder what were to happen if we used some other penalty. For example, what if we instead decided to use the quadratic (Brier) divergence \citep{brier1950}, $\sum_k(w^p_k - w_k)^2$? Second, we argued that $c = 1$ is a natural default option, but one might be interested in knowing what happens if $c$ deviates from 1. Third, we argued for the optimism-penalizing prior (8), but one might wonder how a default option of using a flat prior would fare. 

In Figure \ref{robustness_figures}, we explore what happens in the linear regression experiment from Section 4 when we implement these variations. The top figure shows what happens when $c$ takes a value different from 1. As one might expect, $c=2$ has a better accuracy for smaller sample sizes, whereas $c=0.5$ has a slightly better accuracy for larger sample sizes. In the bottom left figure, we see what happens when we use the Brier divergence rather than KL divergence as the penalty term. The KL divergence performs substantially better. The bottom right figure compares the flat prior with the optimism-penalizing prior. The flat prior results in much worse predictions.

\begin{figure}[h]
        \centering
        \includegraphics[width=0.8\textwidth]{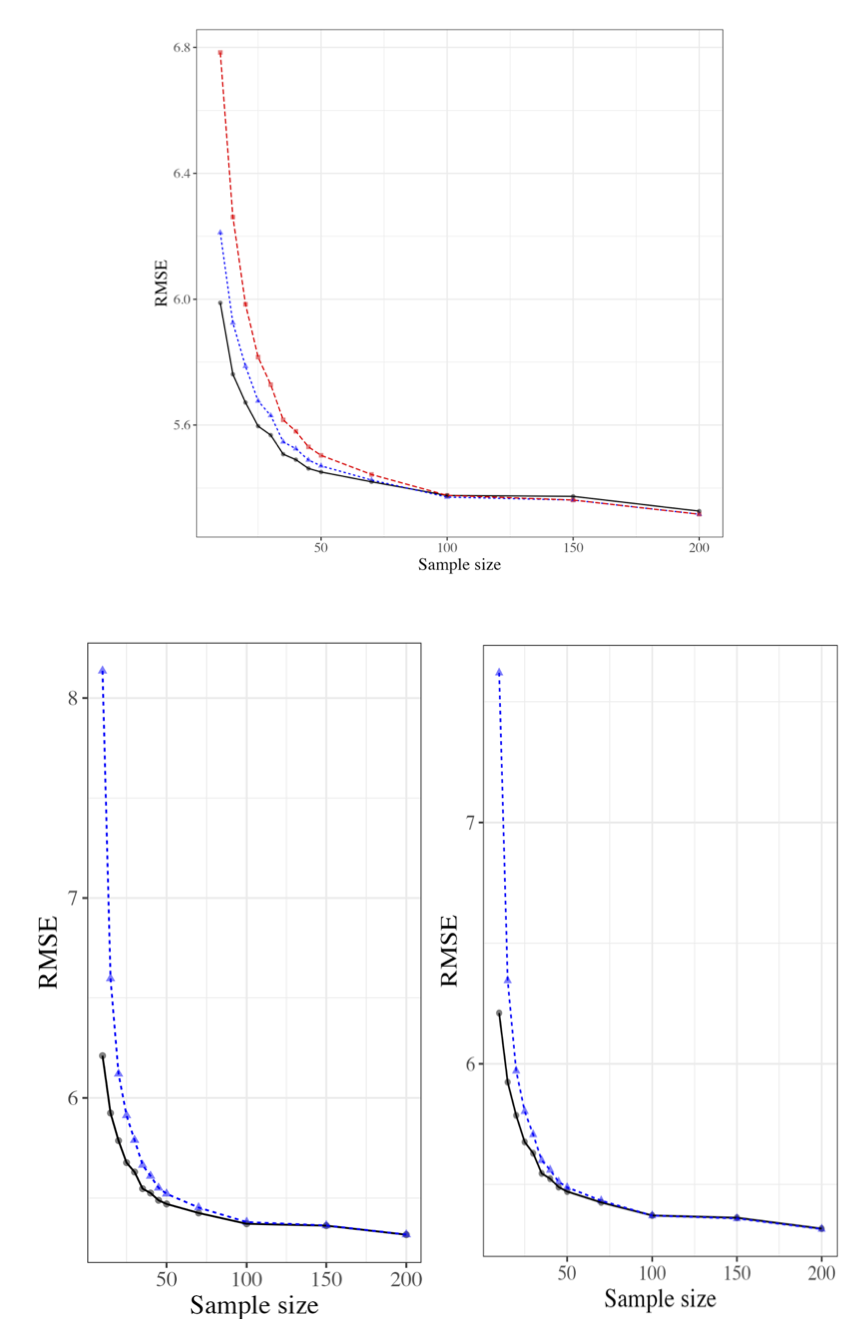}
    \caption{
Top:  c = 2 (black circles); c = 1 (blue triangles); c = 0.5 (red squares) 
Bottom left: Brier divergence (blue triangles); KL divergence (black circles)
Bottom right: Uniform prior (blue triangles); optimism-penalizing prior (black circles)}
 \label{robustness_figures}
\end{figure}

Another robustness check is to consider what happens when we use a data-generating distribution with a heavy-tailed error. Figure \ref{heavy_tails} shows the results after implementing the same linear regression experiment with an error distribution that follows a Student's t-distribution with 3 degrees of freedom.

\begin{figure}[h]
        \centering
        \includegraphics[width=0.8\textwidth]{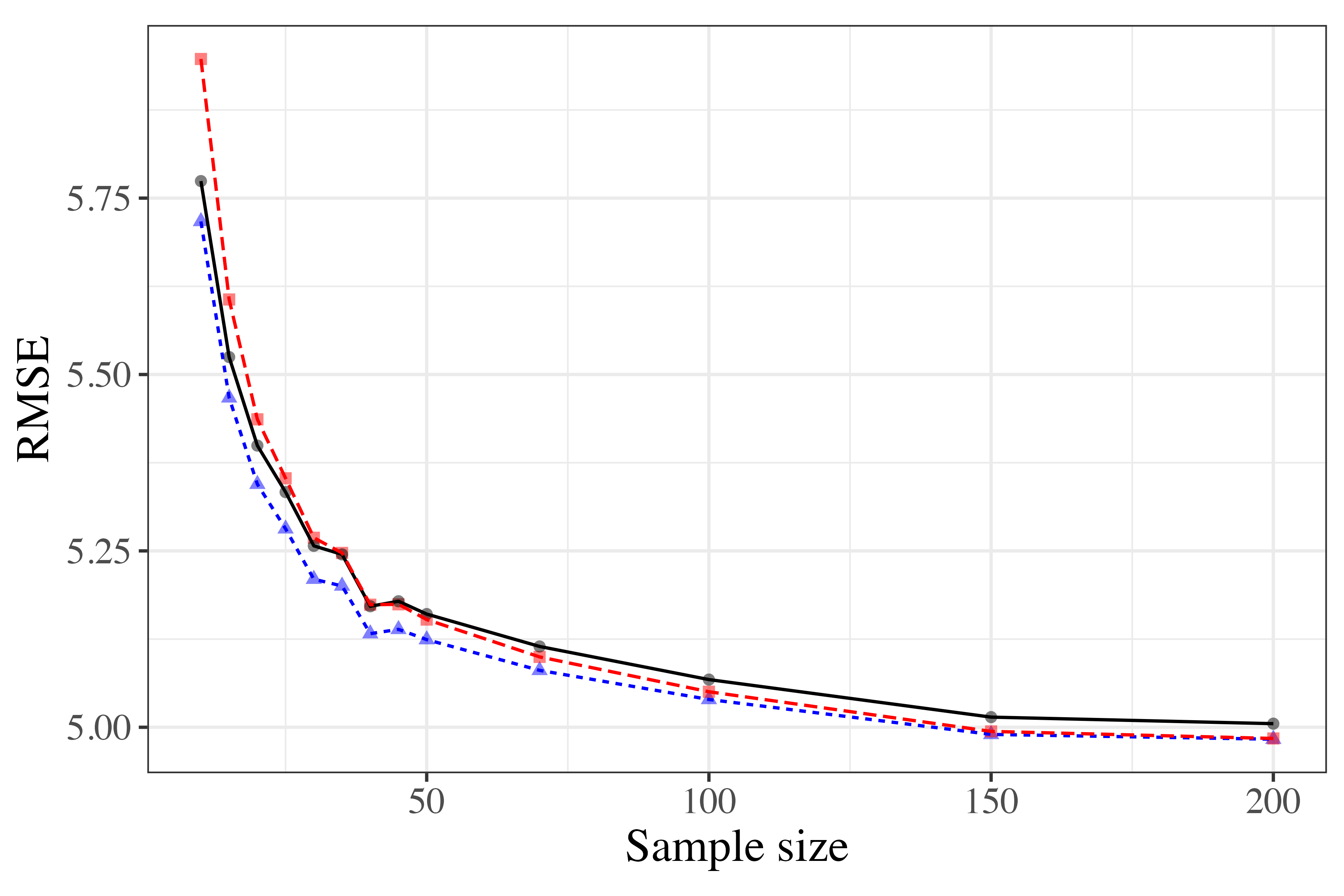}
    \caption{
Blue dotted line: Divergence-based model weighting
Red dotted line: Stacking with the log score
Black solid line: Negative exponentiated model weighting}
 \label{heavy_tails}
\end{figure}

\thispagestyle{empty}

\end{document}